\documentclass[11pt, twoside, letter]{article}
\usepackage[margin = 1in]{geometry}

\usepackage[utf8]{inputenc}
\usepackage{amsmath, amssymb, amsfonts,bm}
\usepackage{ifthen}
\usepackage[boxed, linesnumbered]{algorithm2e}
\usepackage{hyperref}
\usepackage{pgf}
\usepackage{tikz}
\usetikzlibrary{arrows,automata}

\usepackage{microtype}
\usepackage{epsfig}
\usepackage{graphicx}
\usepackage{float}
\usepackage{caption}
\usepackage{subcaption}

\usepackage{amsthm}
\usepackage{multirow, multicol}
\newtheorem{theorem}{Theorem}[section]

\newtheorem{lemma}[theorem]{Lemma}

\newtheorem{remark}[theorem]{Remark}

\newtheorem{definition}[theorem]{Definition}

\numberwithin{equation}{section}

\renewcommand{\vec}[1]{\bm{#1}}

\DeclareMathOperator*{\argmax}{arg\,max}

\newcommand{\EE}{\mathbb{E}}
\newcommand{\PP}{\mathbb{P}}
\newcommand{\norm}[1]{\left\| #1 \right\|}
\usepackage{times}
\usepackage{comment}

\begin{document}

\title{On the Analysis of EM for truncated mixtures of two Gaussians}

\author{Sai Ganesh Nagarajan\\SUTD\\sai\_nagarajan@mymail.sutd.edu.sg
\and Ioannis Panageas\\SUTD\\ioannis@sutd.edu.sg}



\date{}
\maketitle

\begin{abstract}
	Motivated by a recent result of Daskalakis et al. \cite{DGTZ18}, we analyze the population version of Expectation-Maximization (EM) algorithm for the case of \textit{truncated} mixtures of two Gaussians. Truncated samples from a $d$-dimensional mixture of two Gaussians $\frac{1}{2} \mathcal{N}(\vec{\mu}, \vec{\Sigma})+ \frac{1}{2} \mathcal{N}(-\vec{\mu}, \vec{\Sigma})$ means that a sample is only revealed if it falls in some subset $S \subset \mathbb{R}^d$ of positive (Lebesgue) measure. We show that for $d=1$, EM converges almost surely (under random initialization) to the true mean (variance $\sigma^2$ is known) for any measurable set $S$. Moreover, for $d>1$ we show EM almost surely converges to the true mean for any measurable set $S$ when the map of EM has only three fixed points, namely $-\vec{\mu}, \vec{0}, \vec{\mu}$ (covariance matrix $\vec{\Sigma}$ is known), and prove local convergence if there are more than three fixed points. We also provide convergence rates of our findings. Our techniques deviate from those of Daskalakis et al. \cite{DTZ17}, which heavily depend on symmetry that the untruncated problem exhibits. For example, for an arbitrary measurable set $S$, it is impossible to compute a closed form of the update rule of EM. Moreover, arbitrarily truncating the mixture, induces further correlations among the variables. We circumvent these challenges by using techniques from dynamical systems, probability and statistics; implicit function theorem, stability analysis around the fixed points of the update rule of EM and correlation inequalities (FKG).
\end{abstract}

\section{Introduction}

Expectation-Maximization (EM) is an iterative algorithm, widely used to compute the maximum likelihood estimation of parameters in statistical models that depend on hidden (latent) variables $\vec{z}$ (\cite{DLR77}). Given a
probability distribution $p_{\vec{\lambda}}$ defined on $(\vec{x};\vec{z})$, where $\vec{x}$ are the observables and $\vec{\lambda}$ is a parameter vector, and samples $\vec{x}_1,...,\vec{x}_n$, one wants to find $\vec{\lambda}$ in order to maximize the log-likelihood $\sum_{i=1}^n \log p_{\vec{\lambda}}(\vec{x})$ (finite sample case) or $\mathbb{E}_{\vec{x}}[\log p_{\vec{\lambda}}(\vec{x})]$ (population version). Such a task is not always easy, because computing the log-likelihood involves summations over all the possible values of the latent variables and moreover the log-likelihood might be non-concave. EM algorithm is one way to tackle the described problem and works as follows:
\begin{itemize}
\item Guess an initialization of the parameters $\vec{\lambda}_0$.
\item For each iteration:\\ (Expectation-step) Compute the posterior $Q_i(\vec{z})$, which is $p_{\vec{\lambda}_t}(\vec{z} | \vec{x}_i)$ for each sample $i$.\\ (Maximization-step) Compute $\vec{\lambda}_{t+1}$ as the argmax of $\sum_i \sum_{\vec{z}} Q_i (\vec{z}) \log \frac{p_{\vec{\lambda}}(\vec{z},\vec{x}_i)}{Q_i (\vec{z})}$.
\end{itemize}

It is well known that there are guarantees for convergence of EM to stationary points \cite{W83}. The idea behind this fact is that the log-likelihood is decreasing along the trajectories of the EM dynamics.
One of its main applications is on learning a mixture of Gaussians. Recovering the parameters of a mixture of Gaussians with strong guarantees was initiated by Dasgupta \cite{D99} and has been extensively studied in theoretical computer science and machine learning communities, e.g., \cite{AK01}, \cite{KSV05}, \cite{CR08}, \cite{CDV09}, where most of the works assume that the means are \emph{well-separated}. In addition, the authors in \cite{MV10}, \cite{KMV10} offer stronger guarantees, polynomial time (in dimension $d$) learnability of Gaussian mixtures.


For more stylized settings, i.e, for a two component balanced Gaussian mixture with known and equal covariances and unknown means and in some cases symmetric, the following works provide local and global convergence guarantees of both the population version and the finite sample version of EM and remove the \emph{well-separated} assumption.
For instance, authors in \cite{BWY15} study \textbf{local convergence} of the \textbf{mean} parameter in the population and the finite sample version with a \textbf{symmetric mean}, \textbf{balanced} and a \textbf{two component} mixture. However, independent works such as \cite{DTZ17} and \cite{XHM16}, were able to provide \textbf{global convergence} guarantees to the the true \textbf{mean} in the same settings both in the population and the finite sample versions. In the aforementioned cases, the \textbf{population version} is analyzed first as it provides a benchmark for the finite sample regime and that the behavior in finite (but sufficiently large) sample regime cannot deviate too much from the population version. The above results have established there are no spurious local maxima of the log-likelihood in the population version for the above setting. Moreover, the main reason most of the above works are concerned with \textbf{two component} and \textbf{balanced} mixture case as \cite{JZBWJ16} established that the population log-likelihood could have \textbf{spurious} local maxima when there are 3 or more components and that EM converges to it with positive probability. Our work will analyze the same settings as mentioned here.

Parameter estimation problems involving data that has been censored/truncated is crucial in many statistical problems occurring in practice. Statisticians, dating back to Pearson, \cite{PA08} and Fisher, \cite{F31}, tried to address this problem in the early 1900's. Techniques such as method of moments and maximum-likelihood were used for estimating a Gaussian distribution from truncated samples. The seminal work of \cite{R76}, on missing/censored data, tried to approach this by a framework of ignorable and non-ignorable missingness, where the reason for missingness is incorporated into the statistical model through. However, in many cases such flexibilities or the knowledge about the underlying processes may not be available. However, in recent years the focus has shifted towards providing theoretical and computational guarantees for parameter estimation with truncated data. To this end, \cite{DGTZ18} showed that Stochastic Gradient Descent (SGD) converges to the true parameters for a single component Gaussian (in high dimensions) under arbitrary truncation with the assumption that they have oracle access to the truncation set. In our work, we extend the analysis of the EM rule for truncation sets as well as functions (the truncation set can be viewed as a special case of an indicator function).

Finally, mixture models are ubiquitous in machine learning and statistics with a variety of applications ranging from biology \cite{BF08,APMP07} to finance \cite{BM02}. Many of these practical applications are not devoid of some form of truncation or censoring. To this end, there has been previous work that uses EM algorithm for Gaussian mixtures in this setting \cite{LS12}, \cite{MJ88}. However, they assume that truncation sets are boxes and in addition do not provide any convergence guarantees.

\subsection{Our results and techniques}

Our results can be summarized in the following two theorems (one for single-dimensional and one for multi-dimensional case):

\begin{theorem}[Single-dimensional case]\label{thm:single} Let $S \subset \mathbb{R}$ be an arbitrary measurable set of positive Lebesgue measure, i.e, $\int_S \mathcal{N}(x;\mu, \sigma^2) + \mathcal{N}(x;-\mu, \sigma^2) dx = \alpha>0$. It holds that under random initialization (under a measure on $\mathbb{R}$ that is absolutely continuous w.r.t Lebesgue), EM algorithm converges with probability one to either $\mu$ or $-\mu$. Moreover, if initialization $\lambda_0>0$ is in a neighborhood of $\mu$ then EM converges to $\mu$ with an exponential rate \[|\lambda_{t+1}- \mu| \leq \rho_t |\lambda_t - \mu|,\] with $\rho_t = 1 - \Omega(\alpha^4)\min(\alpha^2\min(\lambda_t,\mu),1)$ which is decreasing in $t$. Analogously if $\lambda_0 < 0$, it converges to $-\mu$ with same rate (substitute $\max(\lambda_t,-\mu)$ in the expression).
\end{theorem}

\begin{theorem}[Multi-dimensional case]\label{thm:multi} Let $S \subset \mathbb{R}^d$ with $d>1$ be an arbitrary measurable set of positive Lebesgue measure so that $\int_S \mathcal{N}(\vec{x};\vec{\mu}, \vec{\Sigma}) + \mathcal{N}(\vec{x};\vec{-\mu}, \vec{\Sigma}) d\vec{x} = \alpha>0$. It holds that under random initialization (according to a measure on $\mathbb{R}^d$ that is absolutely continuous with Lebesgue measure), EM algorithm converges with probability one to either $\vec{\mu}$ or $\vec{-\mu}$ as long as EM update rule has only $\vec{-\mu},\vec{0},\vec{\mu}$ as fixed points. Moreover, if initialization $\vec{\lambda}_0$ is in a neighborhood of $\vec{\mu}$ or $-\vec{\mu}$, it converges with a rate $1 - \Omega(\alpha^6)$\footnote{If $ \alpha \mu \ll 1$ then the global convergence rate we provide in the single-dimensional case coincides with the local convergence rate of multidimensional.}.
\end{theorem}

\begin{remark} We would like first to note that we prove the two theorems above in a more general setting where we have truncation functions instead of truncation sets (see Section \ref{sec:model} for definitions). Furthermore, in the proof of Theorem \ref{thm:multi}, we show that $\vec{0}$ is a repelling fixed point and moreover $\vec{-\mu},\vec{\mu}$ are attracting so if the initialization is close enough to $\vec{-\mu}$ or $\vec{\mu}$, then EM actually converges to the true mean. Finally, in Section \ref{sec:multi}, Lemma \ref{lem:symrotation} we provide sufficient conditions of the truncated set $S$ (or truncation function) so that the EM update rule has exactly three fixed points. The sufficient condition is that $S$ is rotation invariant under some appropriate transformation.
\end{remark}

To put our results in the context of recent works discussed above, we see this as the first step in rigorously analyzing the aforementioned settings with truncation which introduces new complexities in the form of \textbf{induced correlations} and as a result our techniques deviate from \cite{DTZ17} and moreover from \cite{DGTZ18} (the latter paper provides mean and covariance estimation of a \textbf{single} high dimensional Gaussian where the likelihood is convex unlike the case of mixtures). Our results indicate that even for the \textbf{two component} mixtures the \textbf{population} version could have \textbf{spurious} fixed points (unlike the untruncated case) even in the simplest case of truncation (rectangles in 2 dimensions) which makes giving global rates challenging. Moreover, we feel that this could also complement \textbf{experimental} results such as \cite{LS12} where they provide a heuristic for box truncated multi-component mixtures but do not provide any theoretical guarantees of convergence.

\paragraph{Technical Overview}

To prove the qualitative part of our two main theorems, we perform stability analysis on the fixed points $\vec{-\mu},\vec{0},\vec{\mu}$ of the dynamical system that is induced by EM algorithm and moreover show that the update rule is a diffeomorphism. This is a general approach that has appeared in other papers that talk about first-order methods avoiding saddle points (\cite{MPP15}, \cite{LSJR16}, \cite{PP17}, \cite{LPPSJR17}, \cite{DP18} to name a few).

Nevertheless, computing the update rule of EM for a truncated mixture of two Gaussians is not always possible, because the set/function $S$ is not necessarily symmetric around $\vec{0}$ (even for functions). As a result, the techniques of \cite{DTZ17} (for the population version) do not carry over to our case. In particular we can find an \textit{implicit} description of the update rule of the EM.

Finally, by leveraging the \textit{Implicit Function Theorem}, we are able to compute explicitly the Jacobian of the update rule of EM and perform spectral analysis on it (Jacobian is computed at the three fixed points $\vec{-\mu},\vec{0}, \vec{\mu}$). We show that the spectral radius of the Jacobian computed at $\vec{-\mu},\vec{\mu}$ is less than one (the fixed points are attracting locally) and moreover the spectral radius of the Jacobian computed at $\vec{0}$ is greater than one (repelling). Along with the fact that the Jacobian is invertible (hence the update rule of EM is a diffeomorphism\footnote{A function is called a diffeomorphism if it is differentiable and a bijection and its inverse is differentiable.}), we can use the center-stable manifold theorem to show that the region of attraction of fixed point $\vec{0}$ is of measure zero. Due to the fact that EM converges always to stationary points (folklore), our result follows. We note that in the case $d=1$, the fixed points are exactly three ($-\mu,0,\mu$) and we prove this fact using FKG (see Theorem \ref{thm:FKG}) inequality. As far as the case $d>1$ is concerned, if $S$ is rotation invariant (under proper transformation so that covariance matrix becomes identity), we can show that there are exactly three fixed points by reducing it to the single dimensional case. Last but not least, for the rates of convergence (quantitative part of our theorems), we prove a quantitative version of the FKG inequality (see Lemma \ref{lem:FKGrevise}) which also might be of independent interest. 
\section{Background}

\subsection{Truncated Mixture Model}\label{sec:model}
Before describing the model, we establish the notations used in this paper. We use bold font to represent vectors, any generic element in $\mathbb{R}^d$ is represented by $\vec{x}$.

The density of a balanced mixture of two different Gaussians with parameters $\left(\vec{\mu}_1,\vec{\Sigma}_1\right)$ and  $\left(\vec{\mu}_2,\vec{\Sigma}_2\right)$ respectively, is given by
$f(\vec{x}) := \dfrac{1}{2}\mathcal{N}(\vec{x};\vec{\mu}_1,\vec{\Sigma}_1)+\frac{1}{2}\mathcal{N}(\vec{x};\vec{\mu}_2,\vec{\Sigma}_2),$ where $\mathcal{N}(\vec{x};\vec{\mu},\Sigma) := \dfrac{\exp(-\frac{1}{2}(\vec{x}-\vec{\mu})^T\Sigma^{-1}(\vec{x}-\vec{\mu}))}{(2\pi)^{\frac{d}{2}}\det(\Sigma)^{1/2}}$. For this work we consider the case when true covariances are known and they are equal to $\vec{\Sigma}$. The means are assumed to be symmetric around the origin and we represent the true parameters of the distribution to be $\left(-\vec{\mu},\vec{\Sigma}\right)$ and $\left(\vec{\mu},\vec{\Sigma}\right)$.

Thus, we can write the density as follows:
\begin{align}\label{eq:density_simple}
f_{\vec{\mu}}(\vec{x}) := \dfrac{1}{2}\mathcal{N}(\vec{x};-\vec{\mu},\vec{\Sigma})+\frac{1}{2}\mathcal{N}(\vec{x};\vec{\mu},\vec{\Sigma}),
\end{align}

Under this setting we consider a truncation set $S \subset \mathbb{R}^d$, which means that we have access only to the samples that fall in the set $S$ which is of positive measure under the true distribution, i.e., \[\int_{\mathbb{R}^d} (0.5\mathcal{N}(\vec{x};-\vec{\mu},\vec{\Sigma})+0.5\mathcal{N}(\vec{x};-\vec{\mu},\vec{\Sigma}))\mathbf{1}_{S}d\vec{x}= \alpha > 0,\] where $\mathbf{1}_{S}$ is the indicator function for $S$, i.e., if $\vec{x} \in S$ then $\mathbf{1}_{S}(\vec{x})=1$ and is zero otherwise.

Hence we can write the truncated mixture density as follows:
\begin{align}
f_{\vec{\mu},S}(\vec{x})={
	\begin{cases}
	\dfrac{0.5\mathcal{N}(\vec{x};-\vec{\mu},\vec{\Sigma})+0.5\mathcal{N}(\vec{x};\vec{\mu},\vec{\Sigma})}{\int_{S}0.5\mathcal{N}(\vec{x};-\vec{\mu},\vec{\Sigma})+0.5\mathcal{N}(\vec{x};\vec{\mu},\vec{\Sigma}) d\vec{x}} \;\;, \vec{x} \in S\\
	0                                              \qquad\qquad\qquad\qquad\qquad\quad\qquad\qquad\qquad\;, \vec{x} \notin S
	\end{cases}
}
\end{align}

The above definition can be generalized for ``truncation" \textit{functions} too. Let $S : \mathbb{R}^d \to \mathbb{R}$ be a non-negative, bounded by one, measurable function so that $0< \alpha = \int_{\mathbb{R}^d} S(\vec{x}) f_{\vec{\mu}}(\vec{x})d\vec{x}$ (we say nonnegative function $S$ is of ``positive measure" if $S(\vec{x})$ is \textit{not} almost everywhere zero). The truncated mixture then is defined as follows:
\[f_{\vec{\mu},S}(\vec{x})= \dfrac{(0.5\mathcal{N}(\vec{x};-\vec{\mu},\vec{\Sigma})+0.5\mathcal{N}(\vec{x};\vec{\mu},\vec{\Sigma}))S(\vec{x})}{\int_{\mathbb{R}^d}(0.5\mathcal{N}(\vec{x};-\vec{\mu},\vec{\Sigma})+0.5\mathcal{N}(\vec{x};\vec{\mu},\vec{\Sigma}))S(\vec{x}) d\vec{x}}
\]
One can think of $S(\vec{x})$ as the probability to actually see sample $\vec{x}$.
\begin{remark}[Results proven for truncation functions]
Our main Theorems \ref{thm:single} and \ref{thm:multi} provided in the introduction, hold in the general setting where we have non-negative truncation functions $S(\vec{x})$ of ``positive measure". Our proofs are written in the general setting (not only the case of indicator functions).
\end{remark}

We will use the following short hand for the truncated EM density with means $\vec{\mu}$ and truncation set or function $S$ such that
$f_{\vec{\mu},S}(\vec{x})=\dfrac{f_{\vec{\mu}}(\vec{x}) \mathbf{1}_{S}}{\int_{\mathbb{R}^d} f_{\vec{\mu}}(\vec{x})\mathbf{1}_S d\vec{x}}$ or $f_{\vec{\mu},S}(\vec{x})=\dfrac{f_{\vec{\mu}}(\vec{x}) S(\vec{x})}{\int_{\mathbb{R}^d}f_{\vec{\mu}}(\vec{x})S(\vec{x})d\vec{x}}$. Also, we will denote the expected value with respect to the truncated mixture distribution with parameters $-\vec{\lambda}$ and $\vec{\lambda}$ by $\EE_{\vec{\lambda},S}\left[.\right]$.
We conclude the subsection with an important definition that will be needed for the multi-dimensional case.
\begin{definition}[Rotation invariant/Symmetric] We call a ``truncation" function $S(\vec{x})$ rotation invariant if $S(Q\vec{x}) = S(\vec{x})$ for all orthogonal matrices $Q$. It is clear that every rotation invariant ``truncation" function is also even (choose $Q = - \vec{I}$, where $\vec{I}$ denotes the identity matrix). A set $S$ is called rotation invariant if $\mathbf{1}_{S}$ is rotation invariant function and moreover it is called symmetric if $\mathbf{1}_{S}$ is an even function.
\end{definition}

Next, we derive the EM-update rule to estimate the mean under the ``truncated" setting.

\subsection{EM Algorithm}
The EM algorithm tries to maximize a lower bound of the likelihood at every time step. The population version of the update rule to estimate the mean of a truncated balanced Gaussian mixture with symmetric means $(-\vec{\mu},\vec{\mu)}$ and covariance $\vec{\Sigma}$ with truncation set $S$ boils down to:

\begin{align}\label{eq:EM-rule}
	h(\vec{\lambda}_t,\vec{\lambda}):=\EE_{\vec{\mu},S}\left[\tanh(\vec{x}^T\vec{\Sigma}^{-1}\vec{\lambda}_t)\vec{x}^T\vec{\Sigma}^{-1}\right]-\EE_{\vec{\lambda},S}\left[\vec{x}^T\vec{\Sigma}^{-1}\tanh(\vec{x}^T\vec{\Sigma}^{-1}\vec{\lambda})\right]
\end{align}
such that
\begin{align}\label{eq:dyn}
		\vec{\lambda}_{t+1}=\left\{\vec{\lambda}:h(\vec{\lambda}_t,\vec{\lambda})=\vec{0}\right\}.
\end{align}	
For the derivation of the update rule please see supplementary material \ref{app:EM-rule}. We note that the above system, in contrast to the un-truncated setting accommodates an \textit{implicit function} in the update rule and hence we cannot obtain a closed form solution for $\vec{\lambda}_{t+1}$.

\begin{remark}[Fixed Points]
We first characterize the fixed points of the dynamical system given in equation (\ref{eq:dyn}).
We can identify that there are 3 fixed points, namely,  $\vec{\mu}$,$-\vec{\mu}$ and $\vec{0}$, since
\begin{align}
 h(\vec{\mu},\vec{\mu})=\vec{0},\;\;h(-\vec{\mu},-\vec{\mu})=\vec{0} \;\;\textit{and}\;\; h(\vec{0},\vec{0})=\vec{0}
 \end{align}
In general there may be more fixed points in the dynamics for any arbitrary truncation function $S(\vec{x})$ or set $S$ (see Section \ref{sec:more}). However, in the single dimension case we prove that there are only three fixed points (see Lemma \ref{lem:threefixedpoints}). In multi-dimensional ($d>1$) case we can also show that if $S$ is rotation invariant, then there are only three fixed points as well (see Lemma \ref{lem:symrotation}).
\end{remark}

\section{Properties of the EM Update Rule}\label{sec:EMalgo}
In the section we analyze the dynamical system arising from the EM update rule. To this end, we first describe the derivative $\nabla_{\vec{\lambda}_t} \vec{\lambda}_{t+1}$ of the dynamical system, by invoking the \textit{Implicit Function Theorem}. Then we present some derivatives that are essential to characterize the dynamics and argue about the stability of fixed points.

\subsection{Properties of the Dynamics}
We use the \textit{Implicit Function Theorem} to represent the derivative of $\vec{\lambda}_{t+1}$ with respect to $\vec{\lambda}_t$ to analyze the dynamical system around some point say $\vec{\gamma}$.

\begin{flalign}\label{eqn:multi-ratio}
	\nabla_{\vec{\lambda}_t}\vec{\lambda}_{t+1} \Big\vert_{\vec{\gamma}}=\nabla_{\vec{\lambda}_{t+1}}\EE_{\vec{\lambda}_{t+1},S}\left[\vec{x}^T\tanh(\vec{x}^T\vec{\Sigma}^{-1}\vec{\lambda}_{t+1})\right]^{-1} \Big\vert_{\vec{\gamma}}
	\cdot\nabla_{\vec{\lambda}_t}\EE_{\vec{\mu},S}\left[\tanh(\vec{x}^T\vec{\Sigma}^{-1}\vec{\lambda}_t)\vec{x}^T\right]\Big\vert_{\vec{\gamma}}
\end{flalign}

The analogue of the above ratio in the single dimension setting is given by:

\begin{align}\label{eqn:single-ratio}
	\frac{d \lambda_{t+1}}{d \lambda_t}\Big\vert_{\vec{\gamma}}=\frac{\frac{d}{d \lambda_t}\EE_{\mu,S}\left[x\tanh\left(\frac{x\lambda_t}{\sigma^2}\right)\right]\Big\vert_{\lambda_t=\gamma}}{\frac{d}{d \lambda_{t+1}}\EE_{\lambda_{t+1},S}\left[x\tanh\left(\frac{x\lambda_{t+1}}{\sigma^2}\right)\right]\Big\vert_{\lambda_t=\gamma}}
\end{align}

To this end, we state the following lemma which describes certain derivatives of the terms involved in the above ratio to argue about local stability of the fixed points.

\begin{lemma}[Some Useful Derivatives]\label{lem:derivatives}
The following equations hold:
\begin{enumerate}	
 \item
 $\begin{aligned}[t]
	\nabla_{\vec{\lambda}} \EE_{\vec{\lambda},S}\left[\vec{x}^T\tanh(\vec{x}^T\vec{\Sigma}^{-1}\vec{\lambda})\right]&=\vec{\Sigma}^{-1}\EE_{\vec{\lambda},S}\left[\vec{x}\vec{x}^T\right]-\\
	&\vec{\Sigma}^{-1}\EE_{\vec{\lambda},S}\left[\vec{x}\tanh(\vec{x}^T\vec{\Sigma}^{-1}\vec{\lambda})\right]\EE_{\vec{\lambda},S}\left[\vec{x}\tanh(\vec{x}^T\vec{\Sigma}^{-1}\vec{\lambda})\right]^T
\end{aligned}$

\item
$\begin{aligned}[t]
	\nabla_{\vec{\mu}} \EE_{\vec{\mu},S}\left[\vec{x}^T\tanh(\vec{x}^T\vec{\Sigma}^{-1}\vec{\lambda})\right]=\vec{\Sigma}^{-1}\EE_{\vec{\mu},S}\left[\vec{x}\vec{x}^T\tanh(\vec{x}^T\vec{\Sigma}^{-1}\vec{\lambda})\tanh(\vec{x}^T\vec{\Sigma}^{-1}\vec{\mu})\right]\\
	\hspace{1in}-\vec{\Sigma}^{-1}\EE_{\vec{\mu},S}\left[\vec{x}\tanh(\vec{x}^T\vec{\Sigma}^{-1}\vec{\lambda})\right]\EE_{\vec{\mu},S}\left[\vec{x}\tanh(\vec{x}^T\vec{\Sigma}^{-1}\vec{\mu})\right]^T
\end{aligned}$

\item $\begin{aligned}[t]
	\nabla_{\vec{\lambda}} \EE_{\vec{\mu},S}\left[\vec{x}^T\tanh(\vec{x}^T\vec{\Sigma}^{-1}\vec{\lambda})\right]&=\vec{\Sigma}^{-1}\EE_{\vec{\mu},S}\left[\vec{x}\vec{x}^T\frac{1}{\cosh^2(\vec{x}^T\vec{\Sigma}^{-1}\vec{\lambda})}\right]\\
	&=\vec{\Sigma}^{-1}\EE_{\vec{\mu},S}\left[\vec{x}\vec{x}^T\left(1-\tanh^2(\vec{x}^T\vec{\Sigma}^{-1}\vec{\lambda})\right)\right]
\end{aligned}$

\end{enumerate}
\end{lemma}

\subsection{Two Important Lemmas}
We end the section about the update rule of EM by proving that is well-defined (in the sense that for every $\lambda_t$ there exists a \textit{unique} $\vec{\lambda}_{t+1}$) and moreover, we show that the update rule has Jacobian that is invertible for all $\vec{x} \in \mathbb{R}^d$.
The first Lemma that is needed to argue about global convergence (in case there are three fixed points), with the use of center-stable manifold (as the proof appears in \cite{LPPSJR17}) is the following:
\begin{lemma}[Local Diffeomorphism]\label{lem:localdiff} Let $J$ be the Jacobian of the update rule of the EM dynamics (of size $d \times d$). It holds that $J$ is invertible.
\end{lemma}
\begin{proof}
	
	It suffices to prove that $	\nabla_{\vec{\lambda}} \EE_{\vec{\lambda},S}\left[\vec{x}^T\tanh(\vec{x}^T\vec{\Sigma}^{-1}\vec{\lambda})\right], 	\nabla_{\vec{\lambda}} \EE_{\vec{\mu},S}\left[\vec{x}^T\tanh(\vec{x}^T\vec{\Sigma}^{-1}\vec{\lambda})\right]$ have non zero eigenvalues (thus invertible) for all $\vec{\lambda} \in \mathbb{R}^d$ and hence the result follows by Equation (\ref{eqn:multi-ratio}).
	Observe that
	\begin{align*}
		M:= \mathbb{E}_{\vec{\lambda},S}[\vec{x}\vec{x}^T(1-\tanh^2(\vec{x}^T \vec{\Sigma}^{-1}\vec{\lambda}))] &= Cov\left(\vec{x}\sqrt{1-\tanh^2(\vec{x}^T \vec{\Sigma}^{-1}\vec{\lambda})}, \vec{x}\sqrt{1 -\tanh^2(\vec{x}^T \vec{\Sigma}^{-1}\vec{\lambda})}\right) \\&+ \mathbb{E}_{\vec{\lambda},S}[\vec{x}\sqrt{1-\tanh^2(\vec{x}^T \vec{\Sigma}^{-1}\vec{\lambda})}]\mathbb{E}_{\vec{\lambda},S}[\vec{x}\sqrt{1-\tanh^2(\vec{x}^T \vec{\Sigma}^{-1}\vec{\lambda})}]^T
	\end{align*}
	(where $\vec{x}$ follows a truncated mixture with parameters $\vec{\lambda}, \vec{\Sigma}$ and truncated function $S$ of ``positive measure") which is positive definite (not positive semidefinite) since the function $S$ is of ``positive measure" and $-1<\tanh(y)< 1$ for all $y \in \mathbb{R}$ (otherwise the variables $\vec{x}_1,...,\vec{x}_d$ would live in a lower dimensional subspace). Moreover, from Lemma \ref{lem:derivatives} it is clear that
	\begin{align*}
		\vec{\Sigma} \nabla_{\vec{\lambda}} \EE_{\vec{\lambda},S}\left[\vec{x}^T\tanh(\vec{x}^T\vec{\Sigma}^{-1}\vec{\lambda})\right] - M &= Cov\left(\vec{x}\tanh(\vec{x}^T \vec{\Sigma}^{-1}\vec{\lambda}), \vec{x}\tanh(\vec{x}^T \vec{\Sigma}^{-1}\vec{\lambda})\right),
	\end{align*}
	which is positive definite as well. Hence we conclude that $$\vec{\Sigma} \cdot \nabla_{\vec{\lambda}} \EE_{\vec{\lambda},S}\left[\vec{x}^T\tanh(\vec{x}^T\vec{\Sigma}^{-1}\vec{\lambda})\right]$$ is positive definite, thus
	$\nabla_{\vec{\lambda}}\EE_{\vec{\lambda},S}\left[\vec{x}^T\tanh(\vec{x}^T\vec{\Sigma}^{-1}\vec{\lambda})\right]$ is invertible.
	The proof for \\$\nabla_{\vec{\lambda}} \EE_{\vec{\mu},S}\left[\vec{x}^T\tanh(\vec{x}^T\vec{\Sigma}^{-1}\vec{\lambda})\right]$ is simpler since
	\begin{align*}
		\vec{\Sigma} \nabla_{\vec{\lambda}} \EE_{\vec{\mu},S}\left[\vec{x}^T\tanh(\vec{x}^T\vec{\Sigma}^{-1}\vec{\lambda})\right] &= Cov\left(\vec{x}\sqrt{1-\tanh^2(\vec{x}^T \vec{\Sigma}^{-1}\vec{\lambda})}, \vec{x}\sqrt{1 -\tanh^2(\vec{x}^T \vec{\Sigma}^{-1}\vec{\lambda})}\right)
		\\&+\mathbb{E}_{\vec{\mu},S}[\vec{x}\sqrt{1-\tanh^2(\vec{x}^T \vec{\Sigma}^{-1}\vec{\lambda})}]\mathbb{E}_{\vec{\mu},S}[\vec{x}\sqrt{1-\tanh^2(\vec{x}^T \vec{\Sigma}^{-1}\vec{\lambda})}]^T,
	\end{align*}
	(where $\vec{x}$ follows a truncated mixture with parameters $\vec{\mu}, \vec{\Sigma}$ and truncated function $S$ of ``positive measure").
	
\end{proof}	

The second lemma is about the fact that the update rule of EM is well defined.
\begin{lemma}[Well defined]\label{lem:welldefined} Let $\lambda_t \in \mathbb{R}^d$. There exists a unique $\vec{\lambda'}$ such that
\[
\EE_{\vec{\mu},S}\left[\tanh(\vec{x}^T\vec{\Sigma}^{-1}\vec{\lambda}_t)\vec{x}^T\vec{\Sigma}^{-1}\right]=\EE_{\vec{\lambda'},S}\left[\vec{x}^T\vec{\Sigma}^{-1}\tanh(\vec{x}^T\vec{\Sigma}^{-1}\vec{\lambda'})\right].
\]
\end{lemma}

\begin{proof}
	
	Let $H(\vec{w}) = \vec{\Sigma}\EE_{\vec{w},S}\left[\vec{x}^T\vec{\Sigma}^{-1}\tanh(\vec{x}^T\vec{\Sigma}^{-1}\vec{w})\right]$. In the Lemma \ref{lem:localdiff} we showed that $\nabla_{\vec{w}}  H(\vec{w})$ is positive definite since $S$ is of positive measure.
	Assume there exist $\lambda, \tilde{\lambda}$ so that $H(\vec{\lambda}) = H(\vec{\tilde{\lambda}})$. Let $\vec{y}_t = t \vec{\lambda} + (1-t) \vec{\tilde{\lambda}}$ for $t \in [0,1]$. Using standard techniques from calculus and that $\nabla_{\vec{w}}  H(\vec{w})$ is symmetric we get that
	\begin{equation}
		(\vec{\lambda} - \vec{\tilde{\lambda}})^T (H(\vec{\lambda}) - H(\vec{\tilde{\lambda}})) \geq \min_{t \in [0,1]}\lambda_{\min} (\nabla_{\vec{w}}  H(\vec{w}) \big\vert_{\vec{w} = \vec{y}_t}) \norm{\vec{\lambda} - \vec{\tilde{\lambda}}}^2,
	\end{equation}
	where $\lambda_{\min}(A)$ denotes the minimum eigenvalue of matrix $A$. It is clear that the left hand side is zero, and also the matrix $\nabla_{\vec{w}}  H(\vec{w}) \big\vert_{\vec{w} = \vec{y}_t}$ has all its eigenvalues positive for every $t \in [0,1]$ (using the fact that $\nabla_{\vec{w}}  H(\vec{w})$ is positive definite for all $w$ from the proof of Lemma \ref{lem:localdiff} above). We conclude that $\vec{\lambda} = \vec{\tilde{\lambda}}$.
	
\end{proof}

\begin{remark} In this remark, we would like to argue why there is always a $\vec{\lambda}_{t+1}$ such that \[
\EE_{\vec{\mu},S}\left[\tanh(\vec{x}^T\vec{\Sigma}^{-1}\vec{\lambda}_t)\vec{x}^T\vec{\Sigma}^{-1}\right]=\EE_{\vec{\lambda}_{t+1},S}\left[\tanh(\vec{x}^T\vec{\Sigma}^{-1}\vec{\lambda}_{t+1})\vec{x}^T\vec{\Sigma}^{-1}\right].
\]
The reason is that $\vec{\lambda}_{t+1}$ is chosen to maximize a particular quantity. If the gradient of that quantity has no roots, it means that $\norm{\vec{\lambda}_{t+1}}_2$ should be infinity. But the quantity is a concave function (in the proof of Lemma \ref{lem:localdiff} we showed that $- \nabla_{\vec{\lambda}} \EE_{\vec{\lambda},S}\left[\vec{x}^T\tanh(\vec{x}^T\vec{\Sigma}^{-1}\vec{\lambda})\right]\vec{\Sigma}^{-1}$ is negative definite which is the Hessian of the quantity to be maximized), so the maximum should be attained in the interior (i.e., $\lambda_{t+1}$ cannot have $\ell_2$ norm infinity).
\end{remark}

\section{Single Dimensional Convergence}\label{sec:single}
In this section we provide a proof for the qualitative part of Theorem \ref{thm:single}. We mention first an important theorem that will be used for the proofs of both qualitative parts of Theorems \ref{thm:single} and \ref{thm:multi}
\begin{theorem}[FKG inequality]\cite{FKG71}\label{thm:FKG}
Let $f,g : \mathbb{R} \to \mathbb{R}$ be two monotonically increasing functions and $\nu$ any probability measure on $\mathbb{R}$. It holds that
\begin{equation}
\int_{\mathbb{R}} f(x)g(x) d\nu \geq \int_{\mathbb{R}} f(x) d\nu \int_{\mathbb{R}} g(x) d\nu.
\end{equation}
Moreover, in case there is positive mass (of the product measure $\nu \otimes \nu$) on the case $(f(x_1)-f(x_2))(g(x_1)-g(x_2)) > 0$ (where $x_1,x_2$ are two independent samples from $\nu$) then the above inequality is strict.
\end{theorem}
We first perform stability analysis for the fixed points $-\mu,0,\mu$ which is captured in the next Lemma.

\begin{lemma}[Stability in single-dimensional]\label{lem:stability1} It holds that \[\left|\frac{d \lambda_{t+1}}{d \lambda_t} \Big\vert_{\lambda_t=0}\right| >1 \textrm{ and }\left|\frac{d \lambda_{t+1}}{d \lambda_t} \Big\vert_{\lambda_t = \mu, -\mu}\right|<1.\]
\end{lemma}
\begin{proof} Using Lemma \ref{lem:derivatives} and Equation (\ref{eqn:single-ratio}) it holds that
	\begin{equation}\label{eq:singlederivative0}
	\frac{d \lambda_{t+1}}{d \lambda_t} \Big\vert_{\lambda_t=0} = \frac{\mathbb{E}_{\mu,S}[x^2]}{\mathbb{E}_{0,S}[x^2]}.
	\end{equation}
	We consider the function $\mathbb{E}_{t\mu,S}[x^2]$ w.r.t variable $t$. We use the Mean Value theorem and we get that there exists $\xi \in (0,1)$ such that
	\begin{align}
		\mathbb{E}_{\mu,S}[x^2] - \mathbb{E}_{0,S}[x^2] &= \frac{d \mathbb{E}_{t\mu,S}[x^2]}{dt}\big\vert _{t = \xi}
		\\& = \frac{1}{\sigma^2} \left[ \mathbb{E}_{\xi\mu,S}\left[x^3\tanh(\xi\mu x) \right] - \mathbb{E}_{\xi\mu,S}\left[x^2 \right]\mathbb{E}_{\xi\mu,S}\left[x\tanh(\xi\mu x) \right] \right]
	\end{align}
	We shall show that \[\mathbb{E}_{\xi\mu,S}\left[x^3\tanh(\xi\mu x) \right] > \mathbb{E}_{\xi\mu,S}\left[x^2 \right]\mathbb{E}_{\xi\mu,S}\left[x\tanh(\xi\mu x) \right].\]
	
	The proof below is inspired by the proof of FKG inequality (because $x^2, x\tanh(x\xi\mu)$ are increasing for $x \geq 0$ and decreasing for $x < 0$).
	Let $x_1,x_2$ be two independent and identically distributed random variables that follow the distribution of $f_{\xi\mu,S}(x)$. Assume w.l.o.g that $|x_1|>|x_2|$ then it holds that $x_1^2 > x_2^2$ and $x_1 \tanh(x_1 \xi\mu) > x_2 \tanh(x_2 \xi\mu)$ (since $\mu>0$).
	Therefore we get that $(x_1^2 - x_2^2)(x_1 \tanh(x_1 \xi\mu)- x_2 \tanh(x_2 \xi\mu))>0$ (except for a measure zero set where it might be equality).
	
	We conclude that \[\mathbb{E}_{\xi\mu,S}[(x_1^2 - x_2^2)(x_1 \tanh(x_1 \xi\mu)- x_2 \tanh(x_2 \xi\mu))]>0.\]
	From independence and the fact that $x_1,x_2$ are identically distributed, we get that
	\[\mathbb{E}_{\xi\mu,S}[x_1^2 x_2 \tanh(x_2 \xi\mu)] = \mathbb{E}_{\xi\mu,S}[x_2^2 x_1 \tanh(x_1 \xi\mu)] = \mathbb{E}_{\xi\mu,S}[x_1^2] \mathbb{E}_{\xi\mu,S}[x_1 \tanh(x_1 \xi\mu)]\]
	and also \[\mathbb{E}_{\xi\mu,S}[x_1^3 \tanh(x_1 \xi\mu)] = \mathbb{E}_{\xi\mu,S}[x_2^3 \tanh(x_2 \xi\mu)].\]
	
	It occurs that $\mathbb{E}_{\xi\mu,S}\left[x_1^3\tanh(\xi\mu x_1) \right] > \mathbb{E}_{\xi\mu,S}\left[x_1^2 \right]\mathbb{E}_{\xi\mu,S}\left[x_1\tanh(\xi\mu x_1) \right]$\\
	
	thus,
	$\mathbb{E}_{\mu,S}[x^2] > \mathbb{E}_{0,S}[x^2]$ (i.e., the ratio (\ref{eq:singlederivative0}) is greater than 1), namely $0$ is a repelling fixed point.
	
	Moreover, using Lemma \ref{lem:derivatives} and Equation (\ref{eqn:single-ratio}) it holds that
	\begin{equation}\label{eq:singlederivative1}
	\frac{d \lambda_{t+1}}{d \lambda_t} \Big\vert_{\lambda_t=\mu} = \frac{\mathbb{E}_{\mu,S}\left[\frac{x^2}{\sigma^2}(1 - \tanh^2(\frac{x\mu}{\sigma^2}))\right]}{\mathbb{E}_{\mu,S}\left[\frac{x^2}{\sigma^2}\right] - \mathbb{E}_{\mu,S}^2\left[\frac{x}{\sigma}\tanh(\frac{x\mu}{\sigma})\right]}.
	\end{equation}
	Since $S$ (function or set) has positive measure we get that the variance of the random variable $\frac{x}{\sigma}\tanh(\frac{x\mu}{\sigma})$ is positive (otherwise the random variable would be constant with probability one and hence $S$ would be of zero measure), thus
	\begin{equation}
	\mathbb{E}_{\mu,S}\left[\frac{x^2}{\sigma^2}\tanh^2\left(\frac{x\mu}{\sigma}\right)\right] > \mathbb{E}^2_{\mu,S}\left[\frac{x}{\sigma}\tanh\left(\frac{x\mu}{\sigma}\right)\right]
	\end{equation}
	or equivalently
	\begin{equation}\label{eq:help}
	\mathbb{E}_{\mu,S}\left[\frac{x^2}{\sigma^2}\right] - \mathbb{E}_{\mu,S}\left[\frac{x^2}{\sigma^2}\tanh^2\left(\frac{x\mu}{\sigma}\right)\right] < \mathbb{E}_{\mu,S}\left[\frac{x^2}{\sigma^2}\right] - \mathbb{E}^2_{\mu,S}\left[\frac{x}{\sigma}\tanh\left(\frac{x\mu}{\sigma}\right)\right].
	\end{equation}
	By inequality (\ref{eq:help}) we conclude that the ratio (\ref{eq:singlederivative1}) is less than one, hence fixed point $\mu$ is attracting.
	The same proof as in the case for $\mu$ works for the fixed point $-\mu$.
\end{proof}
Next, we provide a proof that for the case $d=1$ (single-dimensional), the update rule of EM has exactly three fixed points ($0,\mu,-\mu$).
\begin{lemma}[Only 3 fixed points for single-dimensional]\label{lem:threefixedpoints}
We consider the update rule of the EM method for the single dimensional case (\ref{eq:EM-rule}). The update rule has only $-\mu,0,\mu$ as fixed points.
\end{lemma}

\begin{proof}
Let $\mu>\lambda>0$ and assume $\lambda$ is a fixed point of the update rule of EM (\ref{eq:EM-rule}). Set $G(\mu,\lambda,S) = \mathbb{E}_{\mu,S}[\frac{x}{\sigma^2}\tanh(\frac{x\lambda}{\sigma^2})]$. It holds that $G(\mu,\lambda,S) = G(\lambda,\lambda,S)$ (by definition of $\lambda$).

It follows from Mean Value theorem that there exists a $\xi \in (\lambda, \mu)$ so that (using also Lemma \ref{lem:derivatives})
\begin{align*}
\frac{G(\mu,\lambda,S) - G(\lambda,\lambda,S)}{\mu - \lambda} = &\mathbb{E}_{\xi,S}\left[\frac{x^2}{\sigma^2}\tanh\left(\frac{x\xi}{\sigma^2}\right)\tanh\left(\frac{x\lambda}{\sigma^2}\right)\right]-\\
&\mathbb{E}_{\xi,S}\left[\frac{x}{\sigma}\tanh\left(\frac{x\xi}{\sigma^2}\right)\right]\cdot \mathbb{E}_{\xi,S}\left[\frac{x}{\sigma}\tanh\left(\frac{x\lambda}{\sigma^2}\right)\right].
\end{align*}

We get that $\frac{x}{\sigma}\tanh(\frac{x\lambda}{\sigma^2}), \frac{x}{\sigma}\tanh(\frac{x\xi}{\sigma^2})$ are increasing functions for $x\geq 0$ and decreasing for $x<0$, so inspired by the proof of FKG inequality \ref{thm:FKG} we shall show that $G(\mu,\lambda,S) - G(\lambda,\lambda,S)>0$ (and reach contradiction).

Let $x_1,x_2$ be two independent and identically distributed random variables that follow the distribution of $f_{\xi,S}(x)$. Assume w.l.o.g that $|x_1|>|x_2|$ then it holds that $\frac{x_1}{\sigma} \tanh(\frac{x_1\lambda}{\sigma^2}) > \frac{x_2}{\sigma} \tanh(\frac{x_2\lambda}{\sigma^2})$ and $\frac{x_1}{\sigma} \tanh(\frac{x_1\xi}{\sigma^2}) > \frac{x_2}{\sigma} \tanh(\frac{x_2\xi}{\sigma^2})$.
Therefore we get that $$\left(\frac{x_1}{\sigma} \tanh\left(\frac{x_1\lambda}{\sigma^2}\right)- \frac{x_2}{\sigma} \tanh\left(\frac{x_2\lambda}{\sigma^2}\right)\right)\cdot \left(\frac{x_1}{\sigma} \tanh\left(\frac{x_1\xi}{\sigma^2}\right)- \frac{x_2}{\sigma} \tanh\left(\frac{x_2\xi}{\sigma^2}\right)\right)>0$$ (except for a measure zero set where it might be equality).

We conclude that $$\mathbb{E}_{\xi,S}\left[\left(\frac{x_1}{\sigma} \tanh\left(\frac{x_1\lambda}{\sigma^2}\right)- \frac{x_2}{\sigma} \tanh\left(\frac{x_2\lambda}{\sigma^2}\right)\right)\cdot \left(\frac{x_1}{\sigma} \tanh\left(\frac{x_1\xi}{\sigma^2}\right)- \frac{x_2}{\sigma} \tanh\left(\frac{x_2\xi}{\sigma^2}\right)\right)\right]>0.$$
From independence and the fact that $x_1,x_2$ are identically distributed, we get that
\[\mathbb{E}_{\xi,S}\left[\frac{x_1 x_2}{\sigma^2} \tanh\left(\frac{x_1 \lambda}{\sigma^2}\right) \tanh\left(\frac{x_2 \xi}{\sigma^2}\right)\right] = \mathbb{E}_{\xi,S}\left[\frac{x_1 x_2}{\sigma^2} \tanh\left(\frac{x_1 \xi}{\sigma^2}\right) \tanh\left(\frac{x_2 \lambda}{\sigma^2}\right)\right]\]
and also \[\mathbb{E}_{\xi,S}\left[\frac{x_1^2}{\sigma^2} \tanh\left(\frac{x_1 \xi}{\sigma^2}\right) \tanh\left(\frac{x_1 \lambda}{\sigma^2}\right)\right] = \mathbb{E}_{\xi,S}\left[\frac{x_2^2}{\sigma^2} \tanh\left(\frac{x_2 \xi}{\sigma^2}\right) \tanh\left(\frac{x_2 \lambda}{\sigma^2}\right)\right].\]

We conclude that $G(\mu,\lambda,S) - G(\lambda,\lambda,S)> 0$. However by assumption that $\lambda$ is a fixed point, it must hold that $G(\mu,\lambda,S) - G(\lambda,\lambda,S)=0$ (contradiction).

The same proof works when $\lambda > \mu >0$. In case $\lambda<0$, the proof is exactly the same with before, using $-\mu$ instead of $\mu$ (with opposite direction on the inequality).
\end{proof}

Using the generic proof of Theorem 2, page 6 from \cite{LPPSJR17} paper, the fact that EM converges to stationary points (which are fixed points of the update rule of EM) and combining it with Lemmas \ref{lem:stability1}, \ref{lem:threefixedpoints} and the Lemma \ref{lem:localdiff} about local diffeomorphism of the update rule, the proof of the qualitative part of Theorem \ref{thm:single} follows.

\section{Multi-Dimensional Convergence}\label{sec:multi}
In this section we prove the qualitative part of Theorem \ref{thm:multi}. The techniques follow similar lines as in the single-dimensional case. We will state the Lemmas that deviate technically from those in Section \ref{sec:single}. The two lemmas below provide stability analysis for the fixed points $\vec{-\mu},\vec{0},\vec{\mu}$.
\begin{lemma}[Stability of $\vec{\mu}$ in multi-dimensional]\label{lem:stability2a} It holds that the spectral radius of $$\nabla_{\vec{\lambda}}\EE_{\vec{\lambda},S}\left[\vec{x}^T\tanh(\vec{x}^T\vec{\Sigma}^{-1}\vec{\lambda})\right]^{-1} \Big\vert_{\vec{\lambda}=\vec{\mu}}\cdot\nabla_{\vec{\lambda}}\EE_{\vec{\mu},S}\left[\tanh(\vec{x}^T\vec{\Sigma}^{-1}\vec{\lambda})\vec{x}^T\right]\Big\vert_{\lambda=\vec{\mu}}$$ (i.e., the Jacobian of the update rule of EM method computed at true mean $\vec{\mu}$) is less than one.
\end{lemma}

\begin{proof}
We set $A:= \mathbb{E}_{\vec{\mu},S}[\vec{x}\vec{x}^T] - \mathbb{E}_{\vec{\mu},S}[\vec{x}\tanh(\vec{x}^T \vec{\Sigma}^{-1}\vec{\mu})]\mathbb{E}_{\vec{\mu},S}[\vec{x}\tanh(\vec{x}^T \vec{\Sigma}^{-1}\vec{\mu})]^T$ and $B:= \mathbb{E}_{\vec{\mu},S}[\vec{x}\vec{x}^T] -\mathbb{E}_{\vec{\mu},S}[\vec{x}\vec{x}^T \tanh^2 (\vec{x}^T \vec{\Sigma}^{-1}\vec{\mu})]$. From the proof of Lemma \ref{lem:localdiff} it follows that both $A,B$ are positive definite. Observe that $A-B$ is also positive definite since
\[A - B = Cov(\vec{x}\tanh(\vec{x}^T \vec{\Sigma}^{-1}\vec{\mu}), \vec{x}\tanh(\vec{x}^T \vec{\Sigma}^{-1}\vec{\mu})),\] and the measure $S$ is positive so the vector $\vec{x}\tanh(\vec{x}^T \vec{\Sigma}^{-1}\vec{\mu})$ does not live in a lower dimensional subspace. Moreover, we get that $\vec{\Sigma}^{-1/2}(A-B)\vec{\Sigma}^{-1/2} = \vec{\Sigma}^{-1/2}A\vec{\Sigma}^{-1/2}-\vec{\Sigma}^{-1/2}B\vec{\Sigma}^{-1/2}$ is also positive definite.
We set $\tilde{A} := \vec{\Sigma}^{-1/2}A\vec{\Sigma}^{-1/2}$ and $\tilde{B} := \vec{\Sigma}^{-1/2}B\vec{\Sigma}^{-1/2}$ ($\tilde{A}, \tilde{B}$ are also positive definite). Using Claim \ref{lem:positive} (stated in the end of the section) we conclude that $\tilde{A}^{-1}(\tilde{A} - \tilde{B}) = \vec{I} - \tilde{A}^{-1}\tilde{B}$ has positive eigenvalues. Thus $C: = \vec{I}-\vec{\Sigma}^{1/2}A^{-1}B\vec{\Sigma}^{-1/2}$ has positive eigenvalues. We conclude that $\vec{\Sigma}^{1/2}A^{-1}B\vec{\Sigma}^{-1/2}$ has eigenvalues less than one. Since $\vec{\Sigma}^{1/2}A^{-1}B\vec{\Sigma}^{-1/2}$ has same eigenvalues as $A^{-1}B$, it follows that $A^{-1}B$ has eigenvalues less than one. Finally, from Lemma \ref{lem:positive} it holds that $A^{-1}B$ has positive eigenvalues. The proof follows since $A^{-1}B = \nabla_{\vec{\lambda}}\EE_{\vec{\lambda},S}\left[\vec{x}^T\tanh(\vec{x}^T\vec{\Sigma}^{-1}\vec{\lambda})\right]^{-1} \Big\vert_{\vec{\lambda}=\vec{\mu}}\cdot\nabla_{\vec{\lambda}}\EE_{\vec{\mu},S}\left[\tanh(\vec{x}^T\vec{\Sigma}^{-1}\vec{\lambda})\vec{x}^T\right]\Big\vert_{\lambda=\vec{\mu}}$.
\end{proof}

The same proof works for the case of $\vec{-\mu}$. Below we provide the stability analysis for $\vec{0}$.
\begin{lemma}[Stability of $\vec{0}$ in multi-dimensional]\label{lem:stability2b} It holds that the spectral radius of $$\nabla_{\vec{\lambda}}\EE_{\vec{\lambda},S}\left[\vec{x}^T\tanh(\vec{x}^T\vec{\Sigma}^{-1}\vec{\lambda})\right]^{-1} \Big\vert_{\vec{\lambda}=\vec{0}}\cdot\nabla_{\vec{\lambda}}\EE_{\vec{\mu},S}\left[\tanh(\vec{x}^T\vec{\Sigma}^{-1}\vec{\lambda})\vec{x}^T\right]\Big\vert_{\lambda=\vec{0}}$$ (i.e., the Jacobian of the update rule of EM method computed at true mean $\vec{\mu}$) is greater than one.
\end{lemma}

\begin{proof} We set $\vec{x} \leftarrow \vec{\Sigma}^{-1/2}\vec{x}, \vec{\mu} \leftarrow \vec{\Sigma}^{-1/2}\vec{\mu}$ and define $S'$ accordingly (transforming $S$). It suffices to prove that the matrix $\mathbb{E}_{\vec{0},S'}^{-1}[\vec{x}\vec{x}^T] \mathbb{E}_{\vec{\mu},S'}[\vec{x}\vec{x}^T]$ has an eigenvalue greater than one (using Lemma \ref{lem:derivatives}). We set $G(t) = \mathbb{E}_{t\vec{\mu},S'}[\vec{x}\vec{x}^T]$ and we get that $$\frac{dG}{dt} = \mathbb{E}_{t\vec{\mu},S'}[\vec{x}\vec{x}^T (\vec{x}^T \vec{\mu}) \tanh(\vec{x}^T t\vec{\mu})].$$
Using the fundamental theorem of calculus we get that
\begin{equation}
G(1)-G(0) =  \int_0^1 \mathbb{E}_{t\vec{\mu},S'}[\vec{x}\vec{x}^T (\vec{x}^T \vec{\mu}) \tanh(\vec{x}^T t\vec{\mu})]dt.
\end{equation}
It holds that
\begin{align*}
\vec{\mu}^T G(1) \vec{\mu} = \vec{\mu}^T G(0) \vec{\mu} +  \int_0^1 \vec{\mu}^T\mathbb{E}_{t\vec{\mu},S'}[\vec{x}\vec{x}^T (\vec{x}^T \vec{\mu}) \tanh(\vec{x}^T t\vec{\mu})]\vec{\mu} dt
\end{align*}
The proof below is inspired by the proof of FKG inequality (because $(\vec{x}^T \vec{\mu})^2, \vec{x}^T \vec{\mu}\tanh(\vec{x}^T t\vec{\mu})$ are increasing for $\vec{x}^T \vec{\mu} \geq 0$ and decreasing for $\vec{x}^T \vec{\mu} < 0$ with respect to $\vec{x}^T \vec{\mu}$ and since $t \geq 0$).
Let $\vec{x}_1,\vec{x}_2$ be two independent and identically distributed random variables that follow the distribution of $f_{t\vec{\mu},S'}(\vec{x})$. Assume w.l.o.g that $|\vec{x}_1^T \vec{\mu}|>|\vec{x}_2^T\vec{\mu}|$ then it holds that $(\vec{x}_1^T \vec{\mu})^2 > (\vec{x}_2^T \vec{\mu})^2$ and $\vec{x}_1^T \vec{\mu} \tanh(t\vec{x}_1^T \vec{\mu}) > \vec{x}_2^T \vec{\mu} \tanh(t\vec{x}_2^T \vec{\mu})$ (since $t \geq 0$).

Therefore we get that $\left[(\vec{x}_1^T \vec{\mu})^2 - (\vec{x}_2^T \vec{\mu})^2\right]\left[\vec{x}_1^T \vec{\mu} \tanh(t\vec{x}_1^T \vec{\mu})- \vec{x}_2^T \vec{\mu} \tanh(t\vec{x}_2^T \vec{\mu})\right]>0$ (except for a measure zero set where it might be equality). Thus
\[\mathbb{E}_{t\vec{\mu},S'}\left \{ \left[(\vec{x}_1^T \vec{\mu})^2 - (\vec{x}_2^T \vec{\mu})^2\right]\left[\vec{x}_1^T \vec{\mu} \tanh(t\vec{x}_1^T \vec{\mu})- \vec{x}_2^T \vec{\mu} \tanh(t\vec{x}_2^T \vec{\mu})\right]\right\}>0.\]
By using the fact that $\vec{x}_1, \vec{x}_2$ are independent and identically distributed, it holds that
\[\mathbb{E}_{t\vec{\mu},S'}  \left[(\vec{x}_1^T \vec{\mu})^3 \tanh(t\vec{x}_1^T \vec{\mu})\right]>\mathbb{E}_{t\vec{\mu},S'}  \left[(\vec{x}_1^T \vec{\mu})^2\right]\mathbb{E}_{t\vec{\mu},S'}  \left[ (\vec{x}_1^T \vec{\mu})\tanh(t\vec{x}_1^T \vec{\mu})\right].\]
Hence, we conclude that
$\vec{\mu}^T (\mathbb{E}_{\vec{\mu},S'} [\vec{x}\vec{x}^T] - \mathbb{E}_{\vec{0},S'} [\vec{x}\vec{x}^T])\vec{\mu}>0$, i.e., the matrix $(\mathbb{E}_{\vec{\mu},S'} [\vec{x}\vec{x}^T] - \mathbb{E}_{\vec{0},S'} [\vec{x}\vec{x}^T])$ has a positive eigenvalue. Since $\mathbb{E}_{\vec{0},S'}^{-1}[\vec{x}\vec{x}^T]$ is positive definite, it holds that $$\mathbb{E}_{\vec{0},S'}^{-1}[\vec{x}\vec{x}^T](\mathbb{E}_{\vec{\mu},S'}[\vec{x}\vec{x}^T] - \mathbb{E}_{\vec{0},S'}[\vec{x}\vec{x}^T])$$
has a positive eigenvalue, i.e., $\mathbb{E}_{\vec{0},S'}^{-1}[\vec{x}\vec{x}^T]\mathbb{E}_{\vec{\mu},S'}[\vec{x}\vec{x}^T] - I$ has a positive eigenvalue.\\ Hence, $\mathbb{E}_{\vec{0},S'}^{-1}[\vec{x}\vec{x}^T]\mathbb{E}_{\vec{\mu},S'}[\vec{x}\vec{x}^T]$ has an eigenvalue greater than one, and the proof is complete.
\end{proof}

The following lemma shows that there are three fixed points in the multi-dimensional case when the function $S'(\vec{x}) = S(\vec{\Sigma}^{1/2}\vec{x})$ is rotation invariant.

\begin{lemma}[Rotation invariance implies three fixed points]\label{lem:symrotation} Let $S': \mathbb{R}^d \to \mathbb{R}$ be a rotation invariant function, where $S'(\vec{x}) = S(\vec{\Sigma}^{1/2}\vec{x})$. It holds that the update rule of EM has exactly three fixed points, i.e., $-\vec{\mu},\vec{0},\vec{\mu}$, for any $d>1$.
\end{lemma}

\begin{proof} Consider the transformation $\vec{x} \leftarrow \vec{\Sigma}^{-1/2}\vec{x}, \vec{\mu} \leftarrow \vec{\Sigma}^{-1/2}\vec{\mu}$ and $S' \leftarrow S$. Assume for the sake of contradiction that there exists another fixed point $\vec{\lambda}\neq \vec{0}$ (after the transformation so that we can consider the covariance matrix to be identity). We may assume without loss of generality that $\vec{\mu}^T\vec{\lambda} \geq 0$ since if $\vec{\lambda}$ is a fixed point of the EM rule, so it is $-\vec{\lambda}$. 

Let $Q$ be an orthogonal matrix so that $Q \vec{\lambda} = \norm{\vec{\lambda}}_2 \vec{e}_1$ and $Q \vec{\mu} = \mu_1 \vec{e}_1 + \mu_2 \vec{e}_2$ where $\vec{e}_1, \vec{e}_2 ,...,\vec{e}_d$ is the classic orthogonal basis of $\mathbb{R}^d$ ($Q$ rotates the space), with $\mu_1 \geq 0$ (by assumption) and $\mu_2 \geq 0$ (by the choice of $Q$). We will show that the equation \[\EE_{\vec{\lambda},S'}\left[\tanh(\vec{x}^T\vec{\lambda})\vec{x}\right]=\EE_{\vec{\mu},S'}\left[\tanh(\vec{x}^T\vec{\lambda})\vec{x}\right]\] holds only for $\vec{\lambda} = \vec{\mu}$ (assuming $\vec{\lambda} \neq \vec{\mu}$ we shall reach a contradiction).

Under the transformation $\vec{y} = Q\vec{x}$ (and because $S'$ is rotation invariant, $|\det(Q)| = 1$, $Q^TQ = QQ^T = \vec{I}$) we get that the fixed point $\vec{\lambda}$ of EM satisfies
\begin{equation}\label{eq:transform}
\EE_{Q\vec{\lambda},S'}\left[\tanh(\vec{y}^TQ\vec{\lambda})Q^T\vec{y}\right]=\EE_{Q\vec{\mu},S'}\left[\tanh(\vec{y}^TQ\vec{\lambda})Q^T\vec{y}\right].
\end{equation}
We multiply by $Q$ both sides in Equation (\ref{eq:transform}) and we conclude that
\begin{equation}\label{eq:rule}
\EE_{\norm{\vec{\lambda}}_2 \vec{e}_1,S'}\left[\tanh(\norm{\vec{\lambda}}_2y_1)\vec{y}\right] = \EE_{Q\vec{\mu},S'}\left[\tanh(\norm{\vec{\lambda}}_2y_1)\vec{y}\right],
\end{equation}
We consider the following two cases:
\begin{itemize}
\item $\mu_2 = 0$. For the rest of this case, we denote by $\vec{y}_{-1}$ the vector $\vec{y}$ by removing coordinate $y_1$.

    We use the notation $f_{\nu} = \dfrac{1}{2}\mathcal{N}(\vec{y};-\vec{\nu},\vec{I})+\frac{1}{2}\mathcal{N}(\vec{y};\vec{\nu},\vec{I})$. By rotation invariance of $S'$, it is true that $S'(y_1,\vec{y}_{-1}) = S'(-y_1, \vec{y}_{-1}) = S'(y_1,-\vec{y}_{-1})$, $S'(-\vec{y}) = S'(\vec{y})$ and because $\tanh(\norm{\vec{\lambda}}_2y_1)y_1$ is an even function we get
\begin{align*}
\EE_{Q\vec{\mu},S'}\left[\tanh(\norm{\vec{\lambda}}_2y_1)y_1\right] &= \frac{ \int_{\mathbb{R}^{d}} \tanh(\norm{\vec{\lambda}}_2y_1)y_1S'(\vec{y}) f_{Q\vec{\mu}}d\vec{y}}{\int_{\mathbb{R}^{d}} S'(\vec{y}) f_{Q\vec{\mu}}d\vec{y}}
\\&= \frac{\int_{\mathbb{R}^{d}} \tanh(\norm{\vec{\lambda}}_2y_1)y_1S'(\vec{y}) \mathcal{N}(\vec{y};Q\vec{\mu},\vec{I})d\vec{y}}{\int_{\mathbb{R}^{d}} S'(\vec{y}) \mathcal{N}(\vec{y};Q\vec{\mu},\vec{I})d\vec{y}}
\\&= \frac{\int_{\mathbb{R}}e^{-\frac{(y_1 - (Q\vec{\mu})_1)^2}{2}}\tanh(\norm{\vec{\lambda}}_2y_1)y_1\int_{\mathbb{R}^{d-1}} S'(\vec{y})\mathcal{N}(\vec{y}_{-1};(Q\vec{\mu})_{-1},\vec{I})d\vec{y}_{-1}dy_1}{\int_{\mathbb{R}}e^{-\frac{(y_1 - (Q\vec{\mu})_1)^2}{2}}\int_{\mathbb{R}^{d-1}} S'(\vec{y}) \mathcal{N}(\vec{y}_{-1};(Q\vec{\mu})_{-1},\vec{I})d\vec{y}_{-1}dy_1}
\\&= \frac{\int_{\mathbb{R}}e^{-\frac{(y_1 - (Q\vec{\mu})_1)^2}{2}}\tanh(\norm{\vec{\lambda}}_2y_1)y_1 r(y_1)dy_1}{\int_{\mathbb{R}}e^{-\frac{(y_1 - (Q\vec{\mu})_1)^2}{2}}r(y_1)dy_1}
\end{align*}
where $r(y_1) = \int_{\mathbb{R}^{d-1}} S'(\vec{y}) \mathcal{N}(\vec{y}_{-1};(Q\vec{\mu})_{-1},\vec{I})d\vec{y}_{-1}$ is an even, non-negative function (of positive measure). Since $\tanh(\norm{\vec{\lambda}}_2y_1)y_1 r(y_1)$ is an even function we conclude that
\begin{align*}
\EE_{Q\vec{\mu},S'}\left[\tanh(\norm{\vec{\lambda}}_2y_1)y_1\right] &= \frac{\int_{\mathbb{R}}e^{-\frac{(y_1 - (Q\vec{\mu})_1)^2}{2}}\tanh(\norm{\vec{\lambda}}_2y_1)y_1 r(y_1)dy_1}{\int_{\mathbb{R}}e^{-\frac{(y_1 - (Q\vec{\mu})_1)^2}{2}}r(y_1)dy_1}\\&= \frac{\int_{\mathbb{R}}f_{(Q\mu)_1}\tanh(\norm{\vec{\lambda}}_2y_1)y_1 r(y_1)dy_1}{\int_{\mathbb{R}}f_{(Q\mu)_1}r(y_1)dy_1}
\\&= \EE_{(Q\mu)_1,r}\left[\tanh(\norm{\vec{\lambda}}_2y_1)y_1\right]
\end{align*}
Therefore we conclude that (since $(Q\mu)_1 = \mu_1$)
\[\EE_{\norm{\vec{\lambda}}_2,r}\left[\tanh(\norm{\vec{\lambda}}_2y_1)y_1\right] = \EE_{\mu_1,r}\left[\tanh(\norm{\vec{\lambda}}_2y_1)y_1\right],\]
namely we have reduced the problem to the single dimensional case. Hence from Lemma \ref{lem:threefixedpoints}, it must hold that $\mu_1 = \norm{\vec{\lambda}}_2$, i.e., $\vec{\lambda} = \vec{\mu}$ (contradiction).
\item $\vec{\mu}_2>0$. We use the same machinery as before; by Equation (\ref{eq:rule}), the fact that $S'(y_1,y_2, \vec{y}_{-1,2}) = S'(y_1,y_2, -\vec{y}_{-1,2}) = S'(-y_1,y_2, \vec{y}_{-1,2}) = S'(y_1,-y_2, \vec{y}_{-1,2})$ and moreover the fact that function $\tanh(\norm{\vec{\lambda}}_2y_1)y_2$ is odd with respect to $y_2$, we conclude  \[\EE_{(\mu_1,\mu_2),r'}\left[\tanh(\norm{\vec{\lambda}}_2y_1)y_2\right]=0,\] where $r'(y_1,y_2)$ is a non-negative function (of positive measure) and bounded by 1, with the property that $r'(y_1,y_2) = r'(-y_1,y_2) = r'(y_1,-y_2) = r'(-y_1,-y_2)$ (reducing it to the two-dimensional case).

We will show that \[\EE_{(\mu_1,\mu_2) ,r'}\left[\tanh(\norm{\vec{\lambda}}_2y_1)y_2\right]>0\] and reach a contradiction. Assume $(z_1,z_2) \in \mathbb{R}^2$ so that $z_1 \cdot z_2 > 0$ and $r'(z_1,z_2) > 0$. from the measure $f_{(\mu_1,\mu_2),r'}$. It suffices to show that $f_{(\mu_1,\mu_2),r'}(-z_1,-z_2) = f_{(\mu_1,\mu_2),r'}(z_1,z_2)>f_{(\mu_1,\mu_2),r'}(-z_1,z_2) = f_{(\mu_1,\mu_2),r'}(z_1,-z_2)$ (by the assumption of positive measure). This reduces to (by the property of $r'$)
\[e^{-\frac{(z_1-\mu_1)^2+(z_2-\mu_2)^2}{2}}+e^{-\frac{(z_1+\mu_1)^2+(z_2+\mu_2)^2}{2}} > e^{-\frac{(-z_1-\mu_1)^2+(z_2-\mu_2)^2}{2}}+e^{-\frac{(-z_1+\mu_1)^2+(z_2+\mu_2)^2}{2}},\]
which after cancelling from both sides the common term $e^{-\frac{z_1^2+z_2^2+\mu_1^2+\mu_2^2}{2}}$ is equivalent to
\[\cosh(z_1\mu_1+z_2\mu_2) > \cosh(z_1\mu_1-z_2\mu_2).\]

In case both $z_1,z_2$ are positive then $|z_1\mu_1+z_2\mu_2| > |z_1\mu_1-z_2\mu_2|$ (since $\mu_1,\mu_2>0$) and the inequality follows. In case both $z_1,z_2$ are negative then again $|-z_1\mu_1-z_2\mu_2| > |-z_1\mu_1+z_2\mu_2|$ (since $\mu_1,\mu_2>0$) and the inequality follows. The proof is complete.
\end{itemize}
\end{proof}

\begin{remark} Let $B_{l,r} = \{\vec{x}:l \leq \norm{\vec{x}}_{\vec{\Sigma}^{-1}} \leq r\}$, where $\norm{\vec{x}}_{\vec{\Sigma}^{-1}}$ captures the Mahalanobis distance of $\vec{x}$ from $\vec{0}$, i.e., $\sqrt{\vec{x}^T \vec{\Sigma}\vec{x}}$ ($\vec{\Sigma}$ is positive definite). We would like to note that EM update rule has exactly three fixed points for any truncation set that is a union of $B_{l_i,r_i}$ for a sequence of intervals $(l_i,r_i)$ \footnote{Observe that rotation invariant sets $S$ include unions of $B_{l_i,r_i}$ where the $\vec{\Sigma}$ is the identity matrix.}.
\end{remark}

\begin{lemma}\label{lem:positive} Let $A,B$ be two positive definite matrices. Then $AB$ has positive eigenvalues
\end{lemma}
\begin{proof} $AB$ has the same eigenvalues as $A^{1/2}BA^{1/2}$ ($A^{1/2}$ is well defined since $A$ is positive definite). But $A^{1/2}BA^{1/2}$ is also positive definite, hence the claim follows.
\end{proof}

\subsection{Existence of more Fixed Points}\label{sec:more}
The previous section proved the existence of three fixed points in the case of rotation invariant truncation set/function for the multi-dimensional setting. In this section, we describe an example in two dimensions where the EM update rule has more than three fixed points.
Consider the following setting where the true parameters are give by
$\vec{\mu}\approx \left[2.534,6.395\right]$, and the truncation set $S$ is a ``rectangle", i.e, a product of intervals such that $x_1 \in \left[1,2\right]$ and $x_2 \in \left[-3,1.5\right]$.
We show that $\lambda=\left[1,0\right]$ is a stationary point that satisfies equations, (*) := $\EE_{\vec{\lambda},S}\left[\tanh(\vec{x}^T\vec{\lambda})\vec{x}\right]-\EE_{\vec{\mu},S}\left[\tanh(\vec{x}^T\vec{\lambda})\vec{x}\right]=\vec{0}$.

\begin{figure}
	\centering
	\begin{subfigure}{.5\textwidth}
		\centering
		\includegraphics[width=\linewidth]{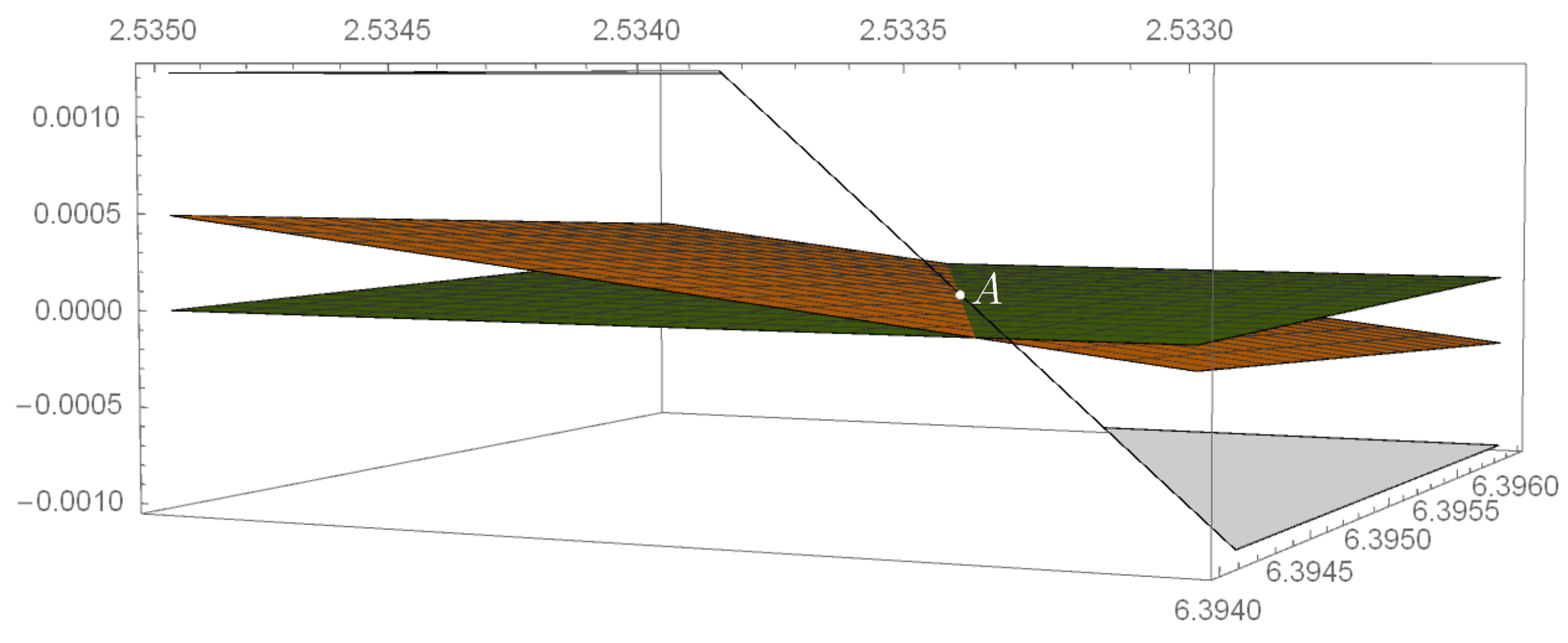}
		\caption{Surfaces of the equations (*) in the neighborhood of fixed point $A$. The point of view is such that the first equation in (*) is a line passing through point $A$.}
		\label{fig:fp_surf}
	\end{subfigure}%
	\begin{subfigure}{.5\textwidth}
		\centering
		\includegraphics[width=0.9\linewidth]{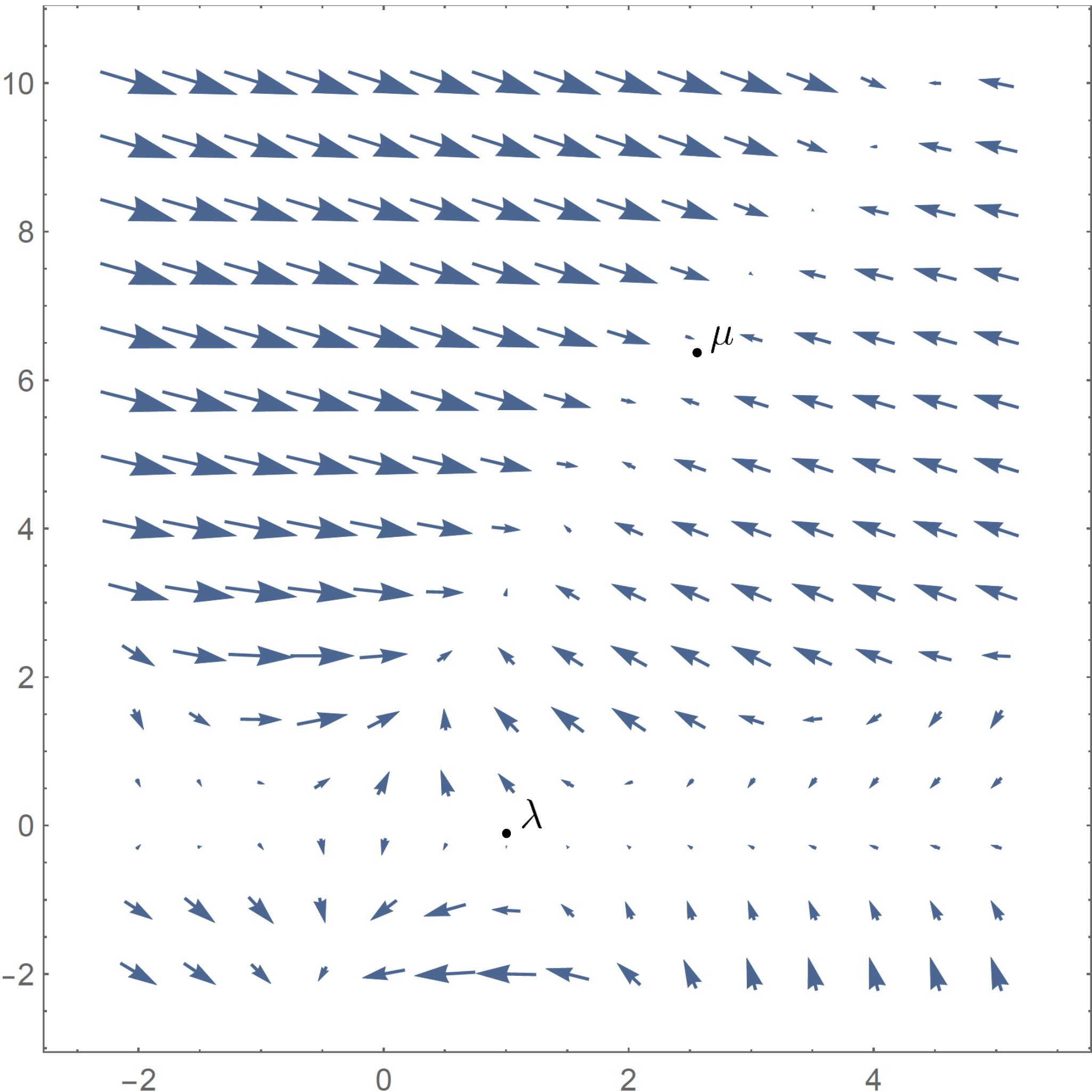}  
		\caption{Vector field of the EM update.}   
		\label{fig:fp_like}
	\end{subfigure}
	\caption{The figures represent the evidence for more fixed points. The figure on the left are the surfaces of the fixed point equation and the figure on the right is the vector field of the EM update.}
	\label{fig:fixed_pts}
\end{figure}
\section{On the Rates of convergence}\label{sec:rates}
In this section we provide quantitative versions of our results of Sections \ref{sec:single} and \ref{sec:multi}.
\subsection{Single-dimensional}
Assume that at iteration $t$ the estimate about the true mean $\lambda_t > 0$. It is easy to see that $\lambda_{t+1}>0$ (the opposite is true if $\lambda_{t}$ is negative). W.l.o.g suppose that $\lambda_t >0$. Moreover, it holds that $\mathbb{E}_{\lambda,S}\left[x \tanh \left(\frac{\lambda x}{\sigma^2}\right)\right],$ $\mathbb{E}_{\mu,S}\left[x \tanh \left(\frac{\lambda x}{\sigma^2}\right)\right]$ are strictly increasing functions in $\lambda$ (argument in the proof of Lemma \ref{lem:localdiff}).

In case $\lambda_t < \mu$ then
\[
\mathbb{E}_{\lambda_{t+1},S}\left[x \tanh \left(\frac{\lambda_{t+1} x}{\sigma^2}\right)\right] = \mathbb{E}_{\mu,S}\left[x \tanh \left(\frac{\lambda_t x}{\sigma^2}\right)\right] > \mathbb{E}_{\lambda_t,S}\left[x \tanh \left(\frac{\lambda_t x}{\sigma^2}\right)\right],
\]
hence $\lambda_{t+1}>\lambda_t$ and moreover since
\[
\mathbb{E}_{\mu,S}\left[x \tanh \left(\frac{\mu x}{\sigma^2}\right)\right]> \mathbb{E}_{\mu,S}\left[x \tanh \left(\frac{\lambda_t x}{\sigma^2}\right)\right] = \mathbb{E}_{\lambda_{t+1},S}\left[x \tanh \left(\frac{\lambda_{t+1} x}{\sigma^2}\right)\right],
\]
it is also true that $\lambda_{t+1}< \mu$. Using the same argument we also conclude that if $\lambda_t > \mu$ then $\lambda_{t} > \lambda_{t+1} >\mu$.

We set $G(\lambda, \mu) = \mathbb{E}_{\mu,S}\left[x \tanh \left(\frac{\lambda x}{\sigma^2}\right)\right]$ and we also assume that $\lambda_t < \mu$. By the mean value theorem, we conclude that
\begin{equation}\label{eq:RHS}
G(\lambda_{t},\mu) - G(\lambda_{t},\lambda_{t}) \geq \min_{\xi \in [\lambda_t, \mu] }\frac{\partial G(\lambda_t,y)}{\partial y}\Bigr|_{y=\xi} (\mu - \lambda_t ).
\end{equation}
Moreover, using mean value theorem again it holds that
\begin{equation}\label{eq:LHS}
G(\lambda_{t+1},\lambda_{t+1}) - G(\lambda_{t},\lambda_{t}) \leq \max_{\xi \in [\lambda_t, \lambda_{t+1}] }\frac{\partial G(y,y)}{\partial y}\Bigr|_{y=\xi} (\lambda_{t+1} - \lambda_{t}).
\end{equation}

Using the fact that $G(\lambda_{t+1},\lambda_{t+1}) = G(\lambda_{t},\mu)$ and Equations (\ref{eq:RHS}), (\ref{eq:LHS}), it follows that
\begin{equation}\label{eq:derivrate}
\max_{\xi \in [\lambda_t, \lambda_{t+1}] }\frac{\partial G(y,y)}{\partial y}\Bigr|_{y=\xi} (\lambda_{t+1} - \lambda_{t}) \geq \min_{\xi \in [\lambda_t, \mu] }\frac{\partial G(\lambda_t,y)}{\partial y}\Bigr|_{y=\xi} (\mu - \lambda_t ).
\end{equation}

By rearranging (\ref{eq:derivrate}) we conclude that $|\lambda_{t+1} - \mu| \leq \left(1 - \frac{\min_{\xi \in [\lambda_t, \mu] }\frac{\partial G(\lambda_t,y)}{\partial y}\Bigr|_{y=\xi}}{\max_{\xi \in [\lambda_t, \lambda_{t+1}] }\frac{\partial G(y,y)}{\partial y}\Bigr|_{y=\xi}} \right)|\lambda_{t} - \mu|$.

In the rest of this section, we will give a lower bound for numerator term $\min_{\xi \in [\lambda_t, \mu] }\frac{\partial G(\lambda_t,y)}{\partial y}\Bigr|_{y=\xi}$ and an upper bound for denominator term $\max_{\xi \in [\lambda_t, \lambda_{t+1}] }\frac{\partial G(y,y)}{\partial y}\Bigr|_{y=\xi}$. As far as denominator is concerned, the following is true.
\begin{lemma}[Bounding the denominator]\label{lem:bounddenom} It holds that \[\frac{\partial G(y,y)}{\partial y}\Bigr|_{y=\xi} \leq O\left( \frac{1}{\alpha^2}\right).\]
\end{lemma}

\begin{proof}
\begin{equation}\label{eq:deryy}
\frac{\partial G(y,y)}{\partial y}\Bigr|_{y=\xi} = \frac{1}{\sigma^2} \left(\mathbb{E}_{\xi,S}[x^2] - \mathbb{E}_{\xi,S}^2\left[x\tanh \left(\frac{x \xi}{\sigma^2}\right)\right]\right).
\end{equation}
Observe now that for each even function $f(x)$ it holds that \[\mathbb{E}_{\xi,S}[f(x)] = \frac{\int_{\mathbb{R}}f(x) \left(e^{-\frac{(x-\xi)^2}{2\sigma^2}}+e^{-\frac{(x+\xi)^2}{2\sigma^2}}\right)S(x) dx}{\int_{\mathbb{R}}\left(e^{-\frac{(x-\xi)^2}{2\sigma^2}}+e^{-\frac{(x+\xi)^2}{2\sigma^2}}\right)S(x) dx} =
\frac{\int_{\mathbb{R}}f(x) e^{-\frac{(x-\xi)^2}{2\sigma^2}}\frac{S(x)+S(-x)}{2} dx}{\int_{\mathbb{R}}e^{-\frac{(x-\xi)^2}{2\sigma^2}}\frac{S(x)+S(-x)}{2} dx},\]
where the last term is just $\mathbb{E}_{\mathcal{N}(\xi,\sigma^2), \frac{S+S'}{2}}[f(x)]$ where $S'(x) = S(-x)$.

We conclude that (\ref{eq:deryy}) becomes
\begin{align*}
\frac{\partial G(y,y)}{\partial y}\Bigr|_{y=\xi} &= \frac{1}{\sigma^2}\left(\mathbb{E}_{\mathcal{N}(\xi,\sigma^2),\frac{S+S'}{2}}[x^2] - \mathbb{E}_{\mathcal{N}(\xi,\sigma^2),\frac{S+S'}{2}}^2\left[x\tanh \left(\frac{x \xi}{\sigma^2}\right)\right]\right)
\\&=\frac{1}{\sigma^2}\left(\mathbb{E}_{\mathcal{N}(\xi,\sigma^2),\frac{S+S'}{2}}[x^2] - \mathbb{E}_{\mathcal{N}(\xi,\sigma^2),\frac{S+S'}{2}}^2\left[x\right]\right)
\\&= \frac{1}{\sigma^2}\mathbb{V}_{\mathcal{N}(\xi,\sigma^2),\frac{S+S'}{2}}[x].
\end{align*}
We use Proposition 1, page 14 along with Lemma 7 in page 13 (for $B$ small enough) from paper \cite{DGTZ18} for truncated Gaussians, it follows that
\begin{equation}
\mathbb{V}_{\mathcal{N}(\xi,\sigma^2),\frac{S+S'}{2}}[x] \leq \mathbb{V}_{\mathcal{N}(\xi,\sigma^2)}\times O\left(\frac{1}{\alpha^2}\right) = \sigma^2 \times  O\left(\frac{1}{\alpha^2}\right).
\end{equation}
The claim follows.
\end{proof}

To bound the numerator, we first provide with the following quantified version of the FKG correlation inequality.
\begin{lemma}[Quantitative FKG]\label{lem:FKGrevise} Let $f,g : \mathbb{R} \to \mathbb{R}$ be two twice continuously differentiable, even functions with $f,g$ are increasing in $(0,+\infty)$ and decreasing in $(-\infty,0)$. Given a random variable $x$, assume with probability at least $q$ it holds that $|x| \geq c>0$ and moreover $|f'(z)| \geq f'(c)$ for all $|z| \geq c$. It holds that
\begin{equation}
\mathbb{E}[f(x)g(x)] - \mathbb{E}[f(x)]\mathbb{E}[g(x)] \geq 2f'(c) g'(c)\cdot q^2\cdot  \mathbb{V}\left[x\Big| \; |x|\geq c\right].
\end{equation}
\end{lemma}

\begin{proof}
Let $y$ be an independent and identically distributed to $x$ random variable. Since both $f,g$ are increasing, we conclude that $(f(x)-f(y))(g(x)-g(y)) \geq 0$ for all possible realizations.
It holds that
\begin{align*}
\mathbb{E}[(f(x)-f(y))(g(x)-g(y))] &\geq \mathbb{E}\left[(f(x)-f(y))(g(x)-g(y)) \Big|\; |x|,|y| \geq c \right] \cdot \Pr[|x| \geq c] \cdot \Pr[|y| \geq c]\\&\geq q^2 \mathbb{E}\left[(f(x)-f(y))(g(x)-g(y)) \Big|\; |x|,|y| \geq c \right] \\& = q^2 \mathbb{E}\left[|f(x)-f(y)||g(x)-g(y)| \Big| \; |x|,|y| \geq c \right]
\\&\geq q^2 f'(c) \cdot g'(c) \mathbb{E}\left[(x-y)^2\Big| \; |x|,|y| \geq c\right].
\end{align*}
The last term, since $x,y$ are independent and identically distributed, is equal to
\begin{align*}
\mathbb{E}\left[(x-y)^2\Big|\; |x|,|y| \geq c\right] &= 2 \mathbb{E}\left[x^2 \Big|\; |x| \geq c\right] - 2 \mathbb{E}^2\left[x \Big|\; |x|\geq c\right]\\& = 2 \mathbb{V}\left[x \Big|\; |x|\geq c\right].
\end{align*}
The proof is complete.
\end{proof}

We are ready to prove a lower bound on the term $\frac{\partial G(\lambda,y)}{\partial y}\Bigr|_{y=\xi}$.
\begin{lemma}[Bounding the numerator]\label{lem:boundnum} It holds that \[\frac{\partial G(\lambda_t,y)}{\partial y}\Bigr|_{y=\xi} \geq \Omega\left(\alpha^2 \tanh^2\left(\sqrt{2\pi}\lambda_t \alpha\right)\right).\]
\end{lemma}

\begin{proof} We will use Lemma \ref{lem:FKGrevise} for the functions $f(x) = x\tanh \left(\frac{\lambda x}{\sigma^2}\right)$ and $g(x) = x\tanh \left(\frac{\xi x}{\sigma^2}\right)$ with $\xi \in [\lambda,\mu]$, $x$ follows $\mathcal{N}(\xi,\sigma^2, \frac{S+S'}{2})$. Moreover we set $q=1/2$ and then the term $c$ from Lemma \ref{lem:FKGrevise} should satisfy $\int_{-c}^c e^{-\frac{(x-\xi)^2}{2\sigma^2}} dx  \leq \frac{\sqrt{2\pi\sigma^2} \alpha}{2}$ for all $\xi \in [\lambda_t, \mu]$. Let $\rho$ be such that $\rho \tanh (\rho) = 1$, it is easy to see that the derivative of $h(x) = x\tanh x$, $|h'(x)| \geq h'(\infty) = 1$ whenever $|x| \geq \rho$, thus if $c \geq \frac{\sigma^2 \rho}{\lambda_t}$ then both $|f'(x)|, |g'(x)| \geq 1$. We assume that $c < \frac{\sigma^2 \rho}{\lambda_t}$.

First observe that $\mathbb{V}\left[x \Big| |x| \geq c\right]$ is the variance of a truncated Gaussian where the truncated measure is at most $\frac{\alpha}{2} + \alpha$ and at least $\alpha$, hence from Lemma 6 and Lemma 7 (with $B$ small enough) from \cite{DGTZ18} we conclude that $\mathbb{V}\left[x \Big| |x| \geq c\right] \geq \Omega(\alpha^2) \times \sigma^2$.
Finally, since $\int_{-c}^c e^{-\frac{(x-\xi)^2}{2\sigma^2}} dx \leq \int_{-c}^c e^{-\frac{x^2}{2\sigma^2}} dx$, we choose $c$ so that
\[\int_{-c}^c e^{-\frac{x^2}{2\sigma^2}} dx < \frac{2c}{\sigma}  = \frac{\alpha \sqrt{2\pi\sigma^2}}{2}.\]

Therefore using Lemma \ref{lem:FKGrevise} and the fact that $\tanh (x) \geq \frac{x}{\cosh^2 x}$ for $x$ positive and $xi \geq \lambda_t$ we conclude that
\begin{align*}
\frac{\partial G(\lambda_t,y)}{\partial y}\Bigr|_{y=\xi} &\geq \frac{1}{4\sigma^2} \left(\tanh \left(\frac{\lambda_t c}{\sigma^2}\right) + \frac{\lambda_tc}{\sigma^2\cosh^2\left(\frac{\lambda_tc}{\sigma^2}\right)}\right)\left(\tanh \left(\frac{\xi c}{\sigma^2}\right) + \frac{\xi c}{\sigma^2\cosh^2\left(\frac{\xi c}{\sigma^2}\right)}\right) \mathbb{V}\left[x \Big| |x| \geq c\right]\\&
\geq \Omega\left(\alpha^2 \tanh^2\left(\sqrt{2\pi}\lambda_t \alpha\right)\right).
\end{align*}
\end{proof}

Combining Lemmas \ref{lem:boundnum}, \ref{lem:bounddenom} along with above discussion, the proof of Theorem \ref{thm:single} is complete.
\subsection{Multi-dimensional}
In this section we prove rates of convergence for the multi-dimensional case when the $\lambda_t$ is sufficiently close to $\vec{\mu}$ or $-\vec{\mu}$. To do this we will prove an upper bound on the spectral radius of the Jacobian $$\nabla_{\vec{\lambda}}\EE_{\vec{\lambda},S}\left[\vec{x}^T\tanh(\vec{x}^T\vec{\Sigma}^{-1}\vec{\lambda})\right]^{-1} \Big\vert_{\vec{\lambda}=\vec{\mu}}\cdot\nabla_{\vec{\lambda}}\EE_{\vec{\mu},S}\left[\tanh(\vec{x}^T\vec{\Sigma}^{-1}\vec{\lambda})\vec{x}^T\right]\Big\vert_{\lambda=\vec{\mu}},$$ i.e., a quantitative version of Lemma \ref{lem:stability2a}.

The following lemma holds and the second part of Theorem \ref{thm:multi} is a corollary.
\begin{lemma}[Rates for local convergence]\label{lem:ratemulti} It holds that the spectral radius of $$\nabla_{\vec{\lambda}}\EE_{\vec{\lambda},S}\left[\vec{x}^T\tanh(\vec{x}^T\vec{\Sigma}^{-1}\vec{\lambda})\right]^{-1} \Big\vert_{\vec{\lambda}=\vec{\mu}, -\vec{\mu}}\cdot\nabla_{\vec{\lambda}}\EE_{\vec{\mu},S}\left[\tanh(\vec{x}^T\vec{\Sigma}^{-1}\vec{\lambda})\vec{x}^T\right]\Big\vert_{\lambda=\vec{\mu}, -\vec{\mu}}$$ (i.e., the Jacobian of the update rule of EM method computed at true mean $\vec{\mu}$) is at most $1 - \Omega(\alpha^6)$.
\end{lemma}
\begin{proof} First we may assume under appropriate transformation ($\vec{x} \leftarrow \vec{\Sigma}^{-1/2}\vec{x}, \vec{\mu} \leftarrow \vec{\Sigma}^{-1/2}\vec{\mu}$) that $\vec{\Sigma} = \vec{I}$. We want to bound the spectral radius of \[\left(\EE_{\vec{\mu},S}\left[\vec{x}\vec{x}^T\right]-\EE_{\vec{\mu},S}\left[\vec{x}\tanh(\vec{x}^T \vec{\mu})\right]\EE_{\vec{\mu},S}\left[\vec{x}\tanh(\vec{x}^T \vec{\mu})\right]^T\right)^{-1}\EE_{\vec{\mu},S}\left[\vec{x}\vec{x}^T(1-\tanh^2(\vec{x}^T\vec{\mu}))\right].\]
	We may assume that $\vec{x}$ follows $\mathcal{N}(\mu,\vec{I}, \frac{S+S'}{2})$ ($S'(\vec{x}) = S(-\vec{x})$), hence we conclude that $\mathbb{E}[\vec{x}\tanh(\vec{x}^T\vec{\mu})] = \mathbb{E}[\vec{x}]$. Thus the Jacobian becomes
	\begin{equation}\label{eq:mati}
	\textrm{Cov}(\vec{x},\vec{x})^{-1} \left(\textrm{Cov}(\vec{x},\vec{x}) - \textrm{Cov}(\vec{x}\tanh(\vec{x}^T \vec{\mu}),\vec{x}\tanh(\vec{x}^T \vec{\mu}))\right).
	\end{equation}
	Using Proposition 1, page 14 and Lemma 7 from page 13 (for small enough $B$) from \cite{DGTZ18} we conclude that $\norm{\textrm{Cov}(\vec{x},\vec{x})}_2$ is at most $O\left(\frac{1}{\alpha^2}\right)$. We choose a $c>0$ such that $\Pr[|\vec{x}^T \vec{\mu}| \geq c] \geq \frac{1}{2}$. By law of Total Variance we get that
	\begin{align*}
		\textrm{Cov}(\vec{x}\tanh(\vec{x}^T\vec{\mu}),\vec{x}\tanh(\vec{x}^T\vec{\mu})) &\succeq \Pr[|\vec{x}^T \vec{\mu}| \geq c] \textrm{Cov}((\vec{x}\tanh(\vec{x}^T\vec{\mu}),\vec{x}\tanh(\vec{x}^T\vec{\mu})) \big| \;|\vec{x}^T \vec{\mu}| \geq c)\\& \succeq
		\frac{\tanh^2 (c)}{2} \textrm{Cov}((\vec{x},\vec{x}) \big| \; |\vec{x}^T \vec{\mu}| \geq c) \\& \succeq \Omega(a^4) \vec{I},
	\end{align*}
	where the last relation holds because of Proposition 1, page 14 and Lemma 7 from page 13 (for small enough $B$) from \cite{DGTZ18} and the fact that $\tanh(c)$ is $\Omega(\alpha)$. Hence the minimum eigenvalue of the matrix above is at least $\Omega(\alpha^4)$. Finally the spectral norm of matrix (\ref{eq:mati}) is at most one minus the minimum eigenvalue of $\textrm{Cov}(\vec{x}\tanh(\vec{x}^T\vec{\mu}),\vec{x}\tanh(\vec{x}^T\vec{\mu}))$ multiplied with the inverse of maximum eigenvalue of $\textrm{Cov}^{-1}(\vec{x},\vec{x})$ hence it is at least $1-\Omega(\alpha^6)$.
\end{proof}

\section{Conclusion}
In this paper, we studied the convergence properties of EM applied to the problem of truncated mixture of two Gaussians. We managed to show that EM converges almost surely to the true mean in the case $d=1$ (with an exponential rate depending on $\alpha$) and moreover that the same result carries over for $d>1$ under the assumption that the update rule of EM has only three fixed points (if it has more, then our results imply local convergence of EM if the initializations are close enough to the true mean). Some interesting questions that arise from this line of work are the following:
\begin{itemize}
	\item Finite population case: Our setting assumes infinite samples. Can we prove a similar convergence result using only finitely many samples? The multi-dimensional case will be challenging because of the existence of more than three fixed points in general.
	\item Beyond two components: Characterize the truncated sets $S$ for which EM converges almost surely to the true mean for truncated mixture of $k$-Gaussians, where $k \geq 3$.
\end{itemize}

\section*{Acknowledgements}
{IP would like to thank Arnab Bhattacharya and Costis Daskalakis for fruitful discussions.}

\bibliographystyle{plain}
\bibliography{colt2019}

\appendix
\appendix
\section{Derivation of EM Update Rule for Truncated Gaussian Mixture} \label{app:EM-rule}
For this derivation, we use $\phi(\vec{x};\bf{\lambda},\vec{\Sigma})$ to denote the normal pdf with mean vector $\vec{\lambda}$ and covariance matrix $\vec{\Sigma}$. Let us denote $c_1$ for the component of the Gaussian corresponding to $+\vec{\lambda}$ and $c_2$ denote the component corresponding to $-\vec{\lambda}$.
First we have to find the posterior densities in the Expectation step. Let $\vec{\lambda}_t$ be our estimate of the parameter at time $t$.

\begin{align}
	\begin{split}
		Q_{\vec{\lambda}_t}(c_1)&=\PP_{\vec{\lambda}_t,S}(Z=c_1|\vec{X}=\vec{x})\\
		&=\frac{\PP_{\vec{\lambda}_t,S}(\vec{X}=\vec{x}|Z=c_1)\PP(Z=c_1)}{\PP_{\vec{\lambda}_t,S}(\vec{X}=\vec{x})}\\
		&=\frac{\phi(\vec{x};\vec{\lambda}_t,\vec{\Sigma})}{\phi(\vec{x};\vec{\lambda}_t,\vec{\Sigma})+\phi(\vec{x};-\vec{\lambda}_t,\vec{\Sigma})}\\
		Q_{\vec{\lambda}_t}(c_2)&=\frac{\phi(\vec{x};-\vec{\lambda}_t,\vec{\Sigma})}{\phi(\vec{x};\vec{\lambda}_t,\vec{\Sigma})+\phi(\vec{x};-\vec{\lambda}_t,\vec{\Sigma})}
	\end{split}
\end{align}
Now the maximization step involves the following:
\begin{align}
	\vec{\lambda}_{t+1}&=\argmax_{\vec{\lambda}}\left[Q_{\vec{\lambda}_t}(c_1)\log\frac{\PP_{\vec{\lambda},S}(\vec{x},c_1)}{Q_{\vec{\lambda}_t}(c_1)}+Q_{\vec{\lambda}_t}(c_2)\log\frac{\PP_{\vec{\lambda},S}(\vec{x},c_2)}{Q_{\vec{\lambda}_t}(c_2)}\right]
\end{align}

Now substituting for $Q_{\vec{\lambda}_t}(c_1)$, $Q_{\vec{\lambda}_t}(c_2)$ and writing $\PP_{\vec{\lambda},S}(\vec{x},c_1)=\frac{\phi(\vec{x};\vec{\lambda},\vec{\Sigma})S(\vec{x})}{\int_{\mathbb{R}^d}(\phi(\vec{x};\vec{\lambda},\vec{\Sigma})+\phi(\vec{x};-\vec{\lambda},\vec{\Sigma}))S(\vec{x})d\vec{x}}$, similarly $\PP_{\vec{\lambda},S}(\vec{x},c_2)=\frac{\phi(\vec{x};-\vec{\lambda},\vec{\Sigma})S(\vec{x})}{\int_{\mathbb{R}^d}(\phi(\vec{x};\vec{\lambda},\vec{\Sigma})+\phi(\vec{x};-\vec{\lambda},\vec{\Sigma}))S(\vec{x})d\vec{x}}$, we get the following:

\begin{align}
	\begin{split}
		\vec{\lambda}_{t+1}&=\argmax_{\vec{\lambda}}\bigg[Q_{\vec{\lambda}_t}(c_1)\log\frac{\phi(\vec{x};\vec{\lambda},\vec{\Sigma})S(\vec{x})}{\int_{\mathbb{R}^d}(\phi(\vec{x};\vec{\lambda},\vec{\Sigma})+\phi(\vec{x};-\vec{\lambda},\vec{\Sigma}))S(\vec{x})d\vec{x}}\\
&+Q_{\vec{\lambda}_t}(c_2)\log\frac{\phi(\vec{x};-\vec{\lambda},\vec{\Sigma})S(\vec{x})}{\int_{\mathbb{R}^d}(\phi(\vec{x};\vec{\lambda},\vec{\Sigma})+\phi(\vec{x};-\vec{\lambda},\vec{\Sigma}))S(\vec{x})d\vec{x}}\bigg]\\
		&=\argmax_{\vec{\lambda}}\bigg[Q_{\vec{\lambda}_t}(c_1)\big(\log\phi(\vec{x};\vec{\lambda},\vec{\Sigma})-\log\int_{\mathbb{R}^d}\phi(\vec{x};\vec{\lambda},\vec{\Sigma})+\phi(\vec{x};-\vec{\lambda},\vec{\Sigma})d\vec{x}\big)\\
		&+Q_{\vec{\lambda}_t}(c_2)\big(\log\phi(\vec{x};-\vec{\lambda},\vec{\Sigma})-\log\int_{\mathbb{R}^d}(\phi(\vec{x};\vec{\lambda},\vec{\Sigma})+\phi(\vec{x};-\vec{\lambda},\vec{\Sigma}))S(\vec{x})d\vec{x}\big)\bigg]\\
		&=\argmax_{\vec{\lambda}}\bigg[Q_{\vec{\lambda}_t}(c_1)\big(\log\phi(\vec{x};\vec{\lambda},\vec{\Sigma})-\log\phi(\vec{x};-\vec{\lambda},\vec{\Sigma})\big)\\
		&+\log\phi(\vec{x};-\vec{\lambda},\vec{\Sigma})-\log\int_{\mathbb{R}^d}(\phi(\vec{x};\vec{\lambda},\vec{\Sigma})+\phi(\vec{x};-\vec{\lambda},\vec{\Sigma}))S(\vec{x})d\vec{x}\bigg]\\
	\end{split}
\end{align}

Finding the gradient of the above maximization we get the following:

\begin{align}
	\begin{split}
		\nabla_{\vec{\lambda}}g(\vec{\lambda};\vec{x},\vec{\Sigma})&=\frac{d}{d\vec{\lambda}}\bigg[Q_{\vec{\lambda}_t}(c_1)\left(2\vec{x}^T\vec{\Sigma}^{-1}\vec{\lambda}\right)-0.5*\vec{x}^T\vec{\Sigma}^{-1}\vec{x}-0.5*\vec{\lambda}^T\vec{\Sigma}^{-1}\vec{\lambda}-\vec{x}^T\vec{\Sigma}^{-1}\vec{\lambda}\\
		&-\log\int_{\mathbb{R}^d} 2*f_{\vec{\lambda}}(\vec{x})S(\vec{x}) d\vec{x}\bigg]\\
		&=\left(2Q_{\vec{\lambda}_t}(c_1)-1\right)\vec{x}^T\vec{\Sigma}^{-1}-\vec{\lambda}^{T}\vec{\Sigma}^{-1}-\int\frac{d}{d\vec{\lambda}}f_{\vec{\lambda},S}(\vec{x}) d\vec{x}\\
		&=\left(2Q_{\vec{\lambda}_t}(c_1)-1\right)\vec{x}^T\vec{\Sigma}^{-1}-\vec{\lambda}^{T}\vec{\Sigma}^{-1}-\int\frac{d}{d\vec{\lambda}}\log f_{\vec{\lambda}}(\vec{x}) f_{\vec{\lambda},S}(\vec{x}) d\vec{x}\\
		&=\left(2Q_{\vec{\lambda}_t}(c_1)-1\right)\vec{x}^T\vec{\Sigma}^{-1}-\vec{\lambda}^{T}\vec{\Sigma}^{-1}-\EE_{\lambda,S}\left[-\vec{\lambda}^T\vec{\Sigma}^{-1}+\vec{x}^T\vec{\Sigma}^{-1}\tanh(\vec{x}^T\vec{\Sigma}^{-1}\vec{\lambda})\right]\\
		&=\left(2Q_{\vec{\lambda}_t}(c_1)-1\right)\vec{x}^T\vec{\Sigma}^{-1}-\EE_{\lambda,S}\left[\vec{x}^T\vec{\Sigma}^{-1}\tanh(\vec{x}^T\vec{\Sigma}^{-1}\vec{\lambda})\right]\\
		&=\tanh(\vec{x}^T\vec{\Sigma}^{-1}\vec{\lambda}_t)\vec{x}^T\vec{\Sigma}^{-1}-\EE_{\lambda,S}\left[\vec{x}^T\vec{\Sigma}^{-1}\tanh(\vec{x}^T\vec{\Sigma}^{-1}\vec{\lambda})\right]
	\end{split}
\end{align}

Thus under the infinite sample case we have the following EM update rule:

\begin{align}
	\vec{\lambda}_{t+1}=\left\{\vec{\lambda}:h(\vec{\lambda}_t,\vec{\lambda})=\vec{0}\right\}
\end{align}
such that $h(\vec{\lambda}_t,\vec{\lambda}_{t+1})=\vec{0}$, where
\begin{align}
	h(\vec{\lambda}_t,\vec{\lambda}):=\EE_{\vec{\mu},S}\left[\tanh(\vec{x}^T\vec{\Sigma}^{-1}\vec{\lambda}_t)\vec{x}^T\vec{\Sigma}^{-1}\right]-\EE_{\vec{\lambda},S}\left[\vec{x}^T\vec{\Sigma}^{-1}\tanh(\vec{x}^T\vec{\Sigma}^{-1}\vec{\lambda})\right].
\end{align}
\newpage

\section{Computation of Derivatives in Lemma \ref{lem:derivatives}}\label{app:derivatives}
\begin{proof}[Proof of Lemma \ref{lem:derivatives}]
We first compute the derivative (3) as follows:
\begin{align}
	\begin{split}
	\nabla_{\vec{\lambda}} \EE_{\vec{\mu},S}\left[\vec{x}^T\tanh(\vec{x}^T\vec{\Sigma}^{-1}\vec{\lambda})\right]&=
	\EE_{\vec{\mu},S}\left[\vec{x}^T\nabla_{\vec{\lambda}}\tanh(\vec{x}^T\vec{\Sigma}^{-1}\vec{\lambda})\right]\\
	&=\vec{\Sigma}^{-1}\EE_{\vec{\mu},S}\left[\vec{x}\vec{x}^T\frac{1}{\cosh^2(\vec{x}^T\vec{\Sigma}^{-1}\vec{\lambda})}\right]\\
	&=\vec{\Sigma}^{-1}\EE_{\vec{\mu},S}\left[\vec{x}\vec{x}^T\left(1-\tanh^2(\vec{x}^T\vec{\Sigma}^{-1}\vec{\lambda})\right)\right]
   \end{split}
\end{align}

Next, derivative (2) is given by:

\begin{align}
	\begin{split}
	\nabla_{\vec{\mu}} \EE_{\vec{\mu},S}\left[\vec{x}^T\tanh(\vec{x}^T\vec{\Sigma}^{-1}\vec{\lambda})\right]&=\nabla_{\vec{\mu}}\dfrac{\int_{\mathbb{R}^d} \vec{x}^T\tanh(\vec{x}^T\vec{\Sigma}^{-1}\vec{\lambda})f_{\vec{\mu}}(\vec{x})S(\vec{x})d\vec{x}}{\int_{\mathbb{R}^d} f_{\vec{\mu}}(\vec{x}) S(\vec{x})d\vec{x}}\\
	&=\dfrac{\int_{\mathbb{R}^d}\vec{x}^T\tanh(\vec{x}^T\vec{\Sigma}^{-1}\vec{\lambda})\nabla_{\vec{\mu}}f_{\vec{\mu}}(\vec{x})S(\vec{x})d\vec{x}}{\int_{\mathbb{R}^d} f_{\vec{\mu}}(\vec{x}) S(\vec{x})d\vec{x}}\\
	-&\dfrac{\int_{\mathbb{R}^d} \nabla_{\vec{\mu}}f_{\vec{\mu}}(\vec{x})S(\vec{x})d\vec{x}}{\int_{\mathbb{R}^d}f_{\vec{\mu}}(\vec{x})S(\vec{x})d\vec{x}}\EE_{\vec{\mu},S}\left[\vec{x}^T\tanh(\vec{x}^T\vec{\Sigma}^{-1}\vec{\lambda})\right]\\
	&=\dfrac{\int_{\mathbb{R}^d}\vec{x}^T\tanh(\vec{x}^T\vec{\Sigma}^{-1}\vec{\lambda})\vec{\Sigma}^{-1}\left(-\vec{\mu}f_{\vec{\mu}}(\vec{x})+\vec{x}^T\tanh(\vec{x}^T\vec{\Sigma}^{-1}\vec{\mu})f_{\vec{\mu}}(\vec{x})\right)S(\vec{x})d\vec{x}}{\int_{\mathbb{R}^d} f_{\vec{\mu}}(\vec{x})S(\vec{x})d\vec{x}}\\
	&-\vec{\Sigma}^{-1}\dfrac{\int_{\mathbb{R}^d} -\vec{\mu}f_{\vec{\mu}}(\vec{x})S(\vec{x}) + \vec{x}^T\tanh(\vec{x}^T\vec{\Sigma}^{-1}\vec{\mu})f_{\vec{\mu}}(\vec{x})S(\vec{x})d\vec{x}}{\int_{\mathbb{R}^d} f_{\vec{\mu}}(\vec{x})S(\vec{x})d\vec{x}}\EE_{\vec{\mu},S}\left[\vec{x}^T\tanh(\vec{x}^T\vec{\Sigma}^{-1}\vec{\lambda})\right]\\
	&=\vec{\Sigma}^{-1}\EE_{\vec{\mu},S}\left[-\vec{x}\vec{\mu}^T\tanh(\vec{x}^T\vec{\Sigma}^{-1}\vec{\lambda})+\vec{x}\vec{x}^T\tanh(\vec{x}^T\vec{\Sigma}^{-1}\vec{\lambda})\tanh(\vec{x}^T\vec{\Sigma}^{-1}\vec{\mu})\right]\\
	&-\vec{\Sigma}^{-1}\Big[\EE_{\vec{\mu},S}\left[-\vec{x}\vec{\mu}^T\tanh(\vec{x}^T\vec{\Sigma}^{-1}\vec{\lambda})\right]\\
	&+\EE_{\vec{\mu},S}\left[\vec{x}^T\tanh(\vec{x}^T\vec{\Sigma}^{-1}\vec{\lambda})\right]\EE_{\vec{\mu},S}\left[\vec{x}^T\tanh(\vec{x}^T\vec{\Sigma}^{-1}\vec{\mu})\right]\Big]\\
	&=\vec{\Sigma}^{-1}\EE_{\vec{\mu},S}\left[\vec{x}\vec{x}^T\tanh(\vec{x}^T\vec{\Sigma}^{-1}\vec{\lambda})\tanh(\vec{x}^T\vec{\Sigma}^{-1}\vec{\mu})\right]\\
	&-\vec{\Sigma}^{-1}\EE_{\vec{\mu},S}\left[\vec{x}\tanh(\vec{x}^T\vec{\Sigma}^{-1}\vec{\lambda})\right]\EE_{\vec{\mu},S}\left[\vec{x}\tanh(\vec{x}^T\vec{\Sigma}^{-1}\vec{\mu})\right]^T
  \end{split}
\end{align}

Finally, derivative of (1) is computed by using the above two derivatives as follows:

\begin{align}
	\begin{split}
		\nabla_{\vec{\lambda}} \EE_{\vec{\lambda},S}\left[\vec{x}^T\tanh(\vec{x}^T\vec{\Sigma}^{-1}\vec{\lambda})\right]&=	\EE_{\vec{\lambda},S}\left[\vec{x}^T\nabla_{\vec{\lambda}}\tanh(\vec{x}^T\vec{\Sigma}^{-1}\vec{\lambda})\right]\\
		&+\dfrac{\int_{\mathbb{R}^d} \vec{x}^T\tanh(\vec{x}^T\vec{\Sigma}^{-1}\vec{\lambda})\nabla_{\vec{\lambda}}f_{\vec{\lambda}}(\vec{x})S(\vec{x})d\vec{x}}{\int_{\mathbb{R}^d} f_{\vec{\lambda}}(\vec{x})S(\vec{x})d\vec{x}}\\
		&-	\dfrac{\int_{\mathbb{R}^d} \nabla_{\vec{\lambda}}f_{\vec{\lambda}}(\vec{x}) S(\vec{x})d\vec{x}}{\int_{\mathbb{R}^d}f_{\vec{\lambda}}(\vec{x})S(\vec{x})d\vec{x}}\EE_{\vec{\lambda},S}\left[\vec{x}^T\tanh(\vec{x}^T\vec{\Sigma}^{-1}\vec{\lambda})\right]\\
		&=\vec{\Sigma}^{-1}\EE_{\vec{\lambda},S}\left[\vec{x}\vec{x}^T\left(1-\tanh^2(\vec{x}^T\vec{\Sigma}^{-1}\vec{\lambda})\right)\right]\\
		&+\vec{\Sigma}^{-1}\EE_{\vec{\lambda},S}\left[\vec{x}\vec{x}^T\tanh(\vec{x}^T\vec{\Sigma}^{-1}\vec{\lambda})\tanh(\vec{x}^T\vec{\Sigma}^{-1}\vec{\lambda})\right]\\
		&-\vec{\Sigma}^{-1}\EE_{\vec{\lambda},S}\left[\vec{x}\tanh(\vec{x}^T\vec{\Sigma}^{-1}\vec{\lambda})\right]\EE_{\vec{\lambda},S}\left[\vec{x}\tanh(\vec{x}^T\vec{\Sigma}^{-1}\vec{\lambda})\right]^T\\
		&=\vec{\Sigma}^{-1}\EE_{\vec{\lambda},S}\left[\vec{x}\vec{x}^T\right]\\
		&-\vec{\Sigma}^{-1}\EE_{\vec{\lambda},S}\left[\vec{x}\tanh(\vec{x}^T\vec{\Sigma}^{-1}\vec{\lambda})\right]\EE_{\vec{\lambda},S}\left[\vec{x}\tanh(\vec{x}^T\vec{\Sigma}^{-1}\vec{\lambda})\right]^T		
	\end{split}
\end{align}
\end{proof}

\end{document}